%% file: main.tex
\documentclass[10pt,twocolumn,letterpaper]{article}

\usepackage{cvpr}              
\usepackage{algorithmic}
\usepackage{algorithm}
\usepackage{multirow} 
\usepackage{bbding}
\usepackage{pifont}
\usepackage{stfloats}
\usepackage[frozencache,cachedir=minted-cache]{minted}

\usepackage[accsupp]{axessibility}  

\usepackage[normalem]{ulem}
\useunder{\uline}{\ul}{}

\usepackage{listings}

\newtheorem{lemma}{Lemma}
\newtheorem{proof}{Proof}
\input{preamble}

\definecolor{cvprblue}{rgb}{0.21,0.49,0.74}
\usepackage[pagebackref,breaklinks,colorlinks,allcolors=cvprblue]{hyperref}

\def\paperID{7791} 
\def\confName{CVPR}
\def\confYear{2025}

\title{One-for-More: Continual Diffusion Model for Anomaly Detection}

\author{
    {Xiaofan Li}$^1$, Xin Tan$^{1}$, Zhuo Chen$^3$, Zhizhong Zhang$^{1}$\thanks{Corresponding author}, Ruixin Zhang$^5$, Rizen Guo$^6$, \\ 
    Guannan Jiang$^7$, Yulong Chen$^{4}$, Yanyun Qu$^{3,8}$, Lizhuang Ma$^{1,4}$, Yuan Xie$^{1,2}$\thanks{Project leader} \\
    $^1$East China Normal University $^2$Shanghai Innovation Institute\\
    $^3$Xiamen University $^4$Shanghai Jiao Tong University\\
    $^5$Tencent YouTu Lab $^6$Tencent WeChatPay Lab33 $^7$CATL \\
    $^8$Key Laboratory of Multimedia Trusted Perception and Efficient Computing \\Ministry of Education of China\\
    {\tt\small \{funzi\}@stu.ecnu.edu.cn, \{zzzhang, lzma, yxie\}@cs.ecnu.edu.cn} \\ 
}

\begin{document}
\maketitle
\input{sec/abstract}    
\input{sec/main0}
{
    \small
    \bibliographystyle{ieeenat_fullname}
    \bibliography{main}
}

\input{sec/X_suppl}


\end{document}

%% file: preamble.tex
\usepackage[dvipsnames]{xcolor}


%% file: sec/abstract.tex
\begin{abstract}
With the rise of generative models, there is a growing interest in unifying all tasks within a generative framework. Anomaly detection methods also fall into this scope and utilize diffusion models to generate or reconstruct normal samples when given arbitrary anomaly images. However, our study found that the diffusion model suffers from severe ``faithfulness hallucination'' and ``catastrophic forgetting'', which can't meet the unpredictable pattern increments. To mitigate the above problems, we propose a continual diffusion model that uses gradient projection to achieve stable continual learning. Gradient projection deploys a regularization on the model updating by modifying the gradient towards the direction protecting the learned knowledge. But as a double-edged sword, it also requires huge memory costs brought by the Markov process. Hence, we propose an iterative singular value decomposition method based on the transitive property of linear representation, which consumes tiny memory and incurs almost no performance loss. Finally, considering the risk of ``over-fitting'' to normal images of the diffusion model, we propose an anomaly-masked network to enhance the condition mechanism of the diffusion model. For continual anomaly detection, ours achieves first place in 17/18 settings on MVTec and VisA. Code is available at \href{https://github.com/FuNz-0/One-for-More}{https://github.com/FuNz-0/One-for-More}

\end{abstract}

%% file: sec/main0.tex
\section{Introduction}
\label{sec:intro}

Anomaly detection (AD) \cite{wuyao1, promptad} has a wide array of applications in both medical and industrial fields. Traditional methods typically operate on a one-to-one paradigm, where a customized model is trained for a specific category or a kind of industrial product. The one-for-one mode severely limits the generalization of the model. Instead, the one-for-all paradigm proposes that training a single model can perform anomaly detection across all classes. However, in real-world scenarios, the pattern increments are unpredictable, necessitating the model's ability to learn continuously, which we refer to as the on-for-more paradigm. This paper focuses on the task of continuous anomaly detection and proposes a \textbf{C}ontinual \textbf{D}iffusion model for \textbf{A}nomaly \textbf{D}etection (\textbf{CDAD}).

\begin{figure}[t]
  \centering
   \includegraphics[width=1\linewidth]{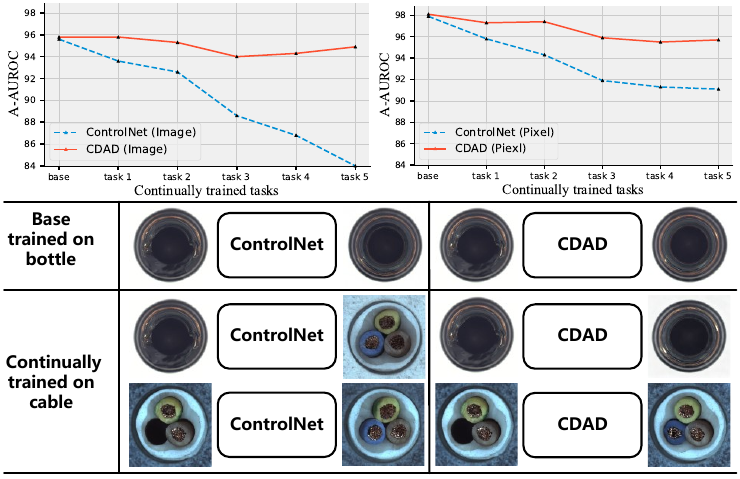}
    \vspace{-15pt}
   \caption{\textbf{Top:} Results of ControlNet and CDAD on image/pixel-level anomaly detection as the number of tasks on the MVTec continually increases. \textbf{Bottom:} ``faithfulness hallucination'' Problem in ControlNet for continuous anomaly detection.}
   \label{fig:intro}
   \vspace{-15pt}
\end{figure}

The diffusion model (DM) has been widely used in anomaly detection \cite{DiffAD, diad, transfusion, GLAD}. Current DM-based methods generate the normal images conditioned on anomaly images through the image-to-image mechanism. The difference between the input and output of the diffusion model is used as the score of anomaly detection, which requires the diffusion model to reconstruct anomaly areas of the input image. However, our study found that the image-to-image diffusion model suffers from ``catastrophic forgetting''\cite{ewc} in continuously generating diverse samples. Taking ControlNet \cite{controlnet} as an example, as shown in Figure \ref{fig:intro} (top), the performance on the base task (10 classes in MVTec \cite{MvTec}) is excellent but deteriorates seriously as the model continues to be trained on new tasks (new classes). Furthermore, as shown in Figure \ref{fig:intro} (bottom), the ControlNet \cite{controlnet}, after training in the latest class, suffers from ``faithfulness hallucination'' \cite{hall}. Specifically, the model's output becomes inconsistent with previously trained samples, leading to a failure in detecting anomalies. In this paper, we mitigate the above problem from two aspects: 1) use a stable parameter update strategy, and 2) strengthen the conditional mechanism of image-to-image generation. 

In this paper, we propose a continual diffusion model framework that projects the updating gradient of new tasks onto a subspace orthogonal to the significant representation of previous tasks \cite{gpm, qiao}, thus eliminating the effect of the update gradient on the previous feature space. Since the significant representation is calculated through singular value decomposition (SVD) on the matrix unfolded by intermediate features, the gradient projection operation requires a lot of memory. However, the Markov-based diffusion process significantly increases the memory overhead. Taking U-Net of stable diffusion \cite{sd} as an example, computing the significant representation of \textbf{10} images consumes almost an additional \textbf{157 GB} of memory (see Table \ref{memory}), and more images need more memory. To address this issue, we propose an iterative singular value decomposition (iSVD) method based on linear transitivity theoretical property. Through the online iterative operation, iSVD only needs about \textbf{17 GB} memory consumption to calculate the significant representation of \textbf{any number} of images. This paper demonstrates and validates the effectiveness of iSVD from both theoretical and experimental perspectives.

Another issue of DM-based anomaly detection methods \cite{DiffAD, diad} is that these methods tend to ``over-fitting'' normal samples rather than focusing on ``reconstructing'' anomalous regions, which weakens the role of the image-to-image conditioning mechanism and aggravates the hallucination problem in continuous learning. To this end, we propose the anomaly-masked network. In it, we use a CNN to encode the input image and match its size with the input of U-Net, and we also design a non-local encoder using a transformer \cite{Transformer} structure to perceive the global information. The neighbor-masked self-attention \cite{uniad} and anomaly-masked loss are used to mask out abnormal features, which preserves normal features output by the local encoder. As a result, the diffusion model pays more attention to the reconstruction of abnormal regions while avoiding its overfitting to generate normal images.  Figure \ref {fig:intro} illustrates our great advantages, compared with ControlNet \cite{controlnet}, CDAD can well overcome the hallucination problem and effectively alleviate the forgetting of continuous anomaly detection. 

To summarize, the main contributions of this paper are:
\begin{itemize}
    \item We propose a continual diffusion framework by modifying the gradient to alleviate the issues of forgetting and hallucination caused by continual learning.

    \item Iterative singular value decomposition is proposed, which can greatly reduce the memory consumption of the gradient projection.

    \item We propose the anomaly-masked network, which can enhance the condition mechanism of image-to-image generation to avoid the side effect of ``over-fitting'' in anomaly detection.
	
    \item For image-level and pixel-level anomaly detection, our method achieves first place in 17/18 continual anomaly detection settings on MVTec \cite{MvTec} and VisA \cite{Visa}.
\end{itemize}

\section{Related Work}

\noindent \textbf{Diffusion models} have achieved significant advancements in the field of image generation \cite{controlnet, sd, dreameboot},which aids other visual tasks \cite{VS-Boost, En-compactness, li2018, ctnet}. They are also widely used in anomaly detection tasks, with some methods \cite{DiffAD, diad, GLAD} utilizing diffusion models to reconstruct anomalous regions of an image. Additionally, some approaches \cite{huteng} use diffusion models to generate various types of defects, which then assist in training anomaly detection models.

\noindent \textbf{One-for-One Anomaly Detection} is the classical anomaly detection framework, which trains a customized model for each class to detect anomalies. The feature embedding methods \cite{PaDiM, PatchCore, SPADE, DifferNet, promptad} extract the patch features of the normal image and store them as memory, then the distance between the test image features and the memory is used as anomaly scores. The knowledge distillation methods \cite{MKD, EfficientAD, FPN, MemKD} compute anomaly scores by the difference between the teacher network and the student network. The reconstruction methods \cite{THFR, FastRecon} train the network to reconstruct abnormal regions into normal regions, and then calculate the anomaly score by the difference. In addition, some methods \citet{Draem, NSA, SimpleNet, unified} assist the model in identifying anomalies by generating anomalies. The one-for-one paradigm has achieved almost saturated performance, but the above methods are unsuitable for multi-class tasks.

\noindent \textbf{One-for-All Anomaly Detection} is a recently emerging multi-class anomaly detection paradigm, which aims to train a model that can detect multiple classes. UniAD \cite{uniad} is the pioneer method, which uses transformer architecture to reconstruct features and proposes neighbor-masked attention to alleviate the ``identical shortcut'' problem. Subsequently, many multi-class anomaly detection methods \cite{mambaad, hvq, he2024learning} have been proposed, which are based on the paradigm of feature reconstruction. DiAD \cite{diad} introduces the latent diffusion model \cite{sd} into multi-class anomaly detection, and proposes the semantic-guided network as the condition mechanism for the generation process. The above methods require all classes to be trained simultaneously, and their performance degrades as the number of classes increases. 

\noindent \textbf{One-for-More Anomaly Detection} is the focus of this paper, which enables models to continuously learn anomaly detection for new categories without backtracking on previous tasks. DNE \cite{dne} and UCAD \cite{ucad} propose using an additional memory bank to store the knowledge of old tasks. The most direct issue is that as tasks increase, the memory bank gradually expands. IUF \cite{iuf} introduces a semantic compression loss to retain the knowledge of previous tasks, however, when one incremental step has many classes, it will still cause a large forgetting of the previous tasks.

\begin{figure*}[t]
  \centering
   \includegraphics[width=1.0\linewidth]{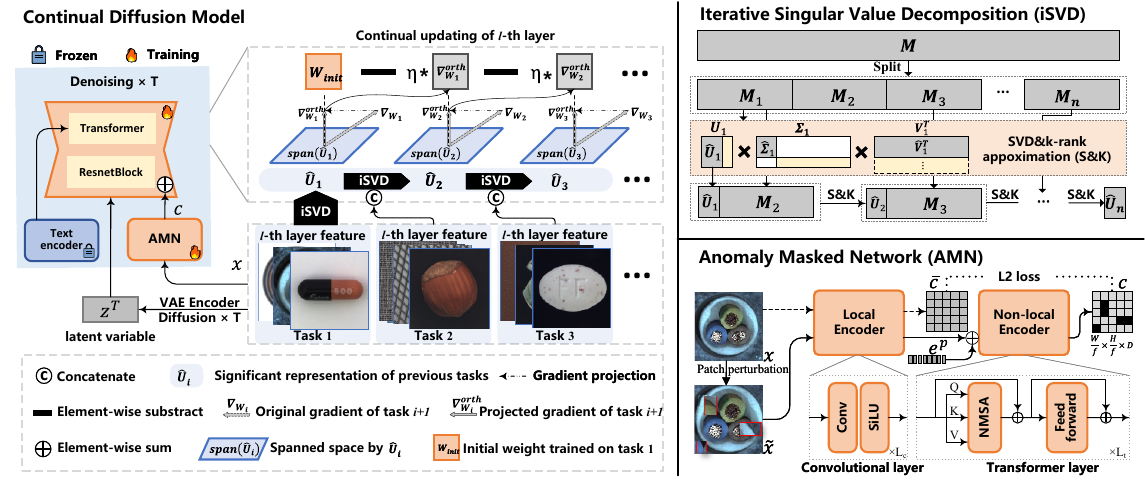}
    
   \caption{Illustration of CDAD. Left is the overall framework of the continual diffusion model. Right is the proposed iterative singular value decomposition and anomaly-masked network.}
   \label{method}
   \vspace{-15pt}
\end{figure*}

\section{Method}

\subsection{Overview}
An overview of our proposed CDAD is illustrated in Figure \ref {method}, our goal is to design a continuously updated diffusion model via task-by-task training to generate or reconstruct various kinds of normal images and then measure the difference between the generated normal image and the test image, to determine whether there exists a defect. The continual diffusion model is based on a pre-trained variational auto-encoder (VAE) and controls the denoise process conditioned on input images. Note that the input of the text encoder is a space character. The continuous update is achieved by modifying the gradient with the direction that has less impact on previous tasks, such that we don't require previous samples and there is no additional inference cost. However, this projection process requires huge memory overhead of computing significant representation, so the iterative singular value decomposition (iSVD) is proposed. In addition, the anomaly-masked network is proposed masking out anomaly features while retaining normal features to serve as conditions for the U-Net. 

During the inference phase, a testing image $x_{test}$ with unknown categories is first encoded into latent space variable by VAE encoder $z^0=\mathcal{E}(x_{test})$ and perform the diffusion process $z^T=Diffusion(z^0)$. Then $x_{test}$ is used as the input of AMN to control denoising processes $\hat{z}=Denoise(z^T, x_{test})$. Finally, the image is reconstructed by VAE decoder $\Tilde{x}_{test}=\mathcal{D}(\hat{z})$. Besides, the original image also reconstructed $\hat{x}_{test}=\mathcal{D}(\mathcal{E}(x_{test}))$. For anomaly localization and detection, we use pre-trained CNNs $\phi(\cdot)$ to extract intermediate-layer feature maps, the distance of $\hat{x}_{test}$ and $\Tilde{x}_{test}$ in position $(h,w)$ of $l$-th layer feature maps can be calculated as:
\begin{equation}
    \begin{aligned}
        \mathcal{M}^l_{h,w}(\hat{x}_{test}, \Tilde{x}_{test})=\left\|\phi^l(\hat{x}_{test})-\phi^l(\Tilde{x}_{test})\right\|_2^2.
    \end{aligned}
\label{eq:ad1}
\end{equation} 
Then, the anomaly score map for anomaly localization is:
\begin{equation}
    \begin{aligned}
        \mathcal{S}=\sum_{l\in \mathbb{L}}Upsample(\mathcal{M}^l(\hat{x}_{test}, \Tilde{x}_{test}))
    \end{aligned}
\label{eq:ad2}
\end{equation} 
Where $\mathbb{L}$ is the set of intermediate layers and $Upsample(\cdot)$ is a bilinear interpolation that can upsample $\mathcal{M}^l$ from different layers to the same size. The anomaly score is the max value of $\mathcal{S}$ for image-level detection.

\subsection{Continual Diffusion Model}
\label{cdm}

The proposed continual diffusion model (CDM) uses a pre-trained VAE to project the image into the latent space for the diffusion and denoising processes. Refer to latent diffusion model \cite{sd}, the objective function is given by:
\begin{equation}
    \begin{aligned}
\mathcal{L}_{C D M}=\mathbb{E}_{\mathcal{E}(x), \Tilde{x}, \epsilon \sim \mathcal{N}(0,I), t}\left\|\epsilon-\epsilon_\theta\left(z_t, t, \tau_\theta(\Tilde{x})\right)\right\|_2^2,
    \end{aligned}
\label{eq:ddpm}
\end{equation}
where $\epsilon_\theta$ is U-net, $\tau_\theta$ is the anomlay-masked network, $\mathcal{E}(\cdot)$ is the encoder of VAE, $z_t$ is the encoded latent variable after $t$ step diffusion, $\Tilde{x}$ is the condition of U-net. Concretely, $\Tilde{x}$ is the original image perturbed by random patches.

Our study found that the diffusion model suffers from ``catastrophic forgetting'' during continual learning. As shown in Figure \ref{method}, the U-Net is mainly composed of ResnetBlock and Transformer, including the CNN and MLP layers. Since MLP and CNN can be represented as matrix multiplication, we can summarize the forgetting principle from the perspective of keeping the previous feature space unchanged after the gradient update. Taking layer $l$ for example and omitting it, denote that $\boldsymbol{X}_{pre(i)} \in \mathbb{R}^{n \times d_1}$ is the input feature for $l$-th layer of previous $i$ tasks, $\boldsymbol{X}_{i} \in \mathbb{R}^{m \times d_1}$ is the input feature of task $i$, $\boldsymbol{W}_{pre} \in \mathbb{R}^{d_1 \times d_2}$ is previous network parameters and $\boldsymbol{O}_{pre} \in \mathbb{R}^{n \times d_2}$ is old output of previous tasks, which is obtained as:
\begin{equation}
    \begin{aligned}
    \boldsymbol{O}_{pre(i)} = \boldsymbol{X}_{pre(i)} \boldsymbol{W}_{pre(i)}, 
    \end{aligned}
\label{eq:gpm1}
\end{equation}
After training on the $i+1$ task, the network parameters are updated, and the output of previous $i$ tasks is as follows:
\begin{equation}
    \begin{aligned}
    \hat{\boldsymbol{O}}_{pre(i)} = \boldsymbol{X}_{pre(i)} \boldsymbol{W}_{i+1}  = \boldsymbol{O}_{pre(i)}  - \eta \boldsymbol{X}_{pre(i)} \nabla_{\boldsymbol{W}_{i}}, 
    \end{aligned}
\label{eq:gpm3}
\end{equation}
where $\eta \neq 0$ is the learning rate, $\nabla_{\boldsymbol{W}_{i}} \in \mathbb{R}^{d_1 \times d_2}$ is the updating gradient of the task $i+1$. After updating, $\hat{\boldsymbol{O}}_{pre} \neq \boldsymbol{O}_{pre}$ which leads to the occurrence of forgetting. To solve this problem, as illustrated in Figure \ref{method} the gradient projection is introduced to project $\nabla_{\boldsymbol{W}_{i+1}}$ into an orthogonal space of $span(\boldsymbol{X}_{pre(i)})$. To obtain this orthogonal space, we first calculate the k-rank column basis $\hat{\boldsymbol{U}}_i \in \mathbb{R}^{d_1 \times k}$ of the previous $i$ tasks is as follows:
\begin{equation}
    \begin{aligned}
		\hat{\boldsymbol{U}}_i = 
		\left\{
		\begin{array}{ll}
			\hat{\mathbf{C}}((\boldsymbol{X}_i)^T), & i = 0, \\ 
			\hat{\mathbf{C}}([\hat{\boldsymbol{U}}_{i-1}, (\boldsymbol{X}_i)^T]), & i > 0, \\
		\end{array}
		\right.
    \end{aligned}
\label{eq:svd}
\end{equation}
where $[\cdot]$ denotes concatenate operation, $\hat{\mathbf{C}}(\boldsymbol{M})$ denotes k-rank column basis of $\boldsymbol{M}$, calculated by SVD followed k-rank approximation. As shown in Figure \ref{method}, for stable continual learning, CDM projects the original gradient of the task $i+1$ onto the orthogonal space of $span(\hat{\boldsymbol{U}}_i)$:
\begin{equation}
    \begin{aligned}
    \nabla_{\boldsymbol{W}_{i}}^{orth} = \nabla_{\boldsymbol{W}_{i}}-(\hat{\boldsymbol{U}}_i)(\hat{\boldsymbol{U}}_i)^T\nabla_{\boldsymbol{W}_{i}},
    \end{aligned}
\label{eq:gpm6}
\end{equation}
where $\boldsymbol{X}_{pre(i)}\nabla_{\boldsymbol{W}_{i}}^{orth} \approx \boldsymbol{0}$. The initial task is updated with the original gradient, and the incremental tasks are updated with the projected gradient. Thus, we can deduce that:
\begin{equation}
    \begin{aligned}
     \hat{\boldsymbol{O}}_{pre(i)} =\boldsymbol{X}_{pre(i)} \boldsymbol{W}_{i} - \eta \boldsymbol{X}_{pre(i)} \nabla_{\boldsymbol{W}_{i}}^{orth} \approx \boldsymbol{O}_{pre(i)},
    \end{aligned}
\label{eq:gpm6}
\end{equation}
where the updated model on the new task does not perturb the output space of the previous task, thus ensuring the model retains the old knowledge.

\begin{algorithm}[t]
    \caption{SVD followed by k-rank approximation}
    \vspace{-10pt}
    \begin{minted}[fontsize=\footnotesize, framesep=0.1mm,]{python}
# M: Feature matrix, d × m
# th: Threshold
def S_and_K(M, th):
    # U: d × d, S: d, Vh: m × m
    U, S, Vh = torch.linalg.svd(M)
    total = (S ** 2).sum()
    ratio = (S ** 2) / total
    k = (torch.cumsum(ratio) < th).sum()
    U_hat = U[:, :k]
    return U_hat
    \end{minted}
    \vspace{-10pt}
    \label{algo1}
\end{algorithm}

\subsection{Iterative Singular Value Decomposition}
\label{isvd}
The traditional SVD method used in Eq (\ref{eq:svd}) needs the global calculation to calculate the significant representation, thus it requires enough memory to store intermediate features. However, the Markov-based denoising process greatly increases the size of intermediate features, according to our statistics, it requires almost 157 GB of memory to compute significant representation for only ten images (see Table \ref{memory}), which is an unaffordable memory cost. Therefore, we propose an iterative SVD (iSVD) method, which can used to calculate the significant representation of a large matrix and retain the original properties with a small memory overhead. Specifically, we denote the input matrix of Eq (\ref{eq:svd}) as $\boldsymbol{M} \in \mathbb{R}^{d \times \Lambda}$, where $d$ and $\Lambda$ denote the dimension and number of features and $\Lambda \gg d$. As illustrated in Figure \ref{method}, $\boldsymbol{M}$ is split into $\{\boldsymbol{M}_1, \boldsymbol{M}_2, \dots, \boldsymbol{M}_n\}$, where $\boldsymbol{M}_i \in \mathbb{R}^{d \times m}$, $m=\frac{\Lambda}{n}$. The significant representation $\hat{\boldsymbol{U}}_1 \in \mathbb{R}^{d \times k}$ of $\boldsymbol{M}_1$ is calculated by SVD followed by its k-rank approximation according to the threshold, $\gamma_{th}$:  
\begin{equation}
\centering
    \begin{aligned}
    \boldsymbol{M}_1 = \boldsymbol{U}_1 \boldsymbol{\Sigma}_1 \boldsymbol{V}_1^T, 
    \boldsymbol{U}_1  & \in \mathbb{R}^{d \times d}, 
    \boldsymbol{\Sigma}_1 \in \mathbb{R}^{d \times m}, \\
    (\boldsymbol{M}_1)_{k} = \hat{\boldsymbol{U}}_1 \hat{\boldsymbol{\Sigma}}_1 \hat{\boldsymbol{V}}_1^T,
    \hat{\boldsymbol{U}}_1 & \in \mathbb{R}^{d \times k}, 
    \hat{\boldsymbol{\Sigma}}_1 \in \mathbb{R}^{k \times k}, \\
    \text{s.t.} \quad \left\|\left(\boldsymbol{M}_1\right)_{k}\right\|_F^2 & \geq \gamma_{th}\left\|\boldsymbol{M}_1\right\|_F^2 ,
    \end{aligned}
\label{eq:gpm4}
\end{equation}
where $ \left\|\cdot\right\|_F^2$ denotes $F$-norm, $\boldsymbol{U}_1=[u^1, u^2, \dots, u^d], u^i \in \mathbb{R}^{d}$, and the significant representation $\hat{\boldsymbol{U}}_1=[u^1, u^2, \dots, u^{k}]$, selected by the $k$ vectors with highest singular values in $\boldsymbol{U}_1$, the pytorch implementation is shown in Algorithm \ref{algo1}.
Then, $\hat{\boldsymbol{U}}_1$ is concatenated with $\boldsymbol{M}_2$ to continue S\&K operation iteratively. Finally, the overall significant representation $\hat{\boldsymbol{U}}_n$ is obtained to calculate the orthogonal gradient of the new task. Next, we will prove that $\hat{\boldsymbol{U}}_n$ obtained through iSVD is a significant representation of $\boldsymbol{M}$ and give the memory saving rating.

\begin{lemma}
    \label{lemma}
    linear transitivity: Given matrices $\boldsymbol{S}_1$, $\boldsymbol{S}_2$, and $\boldsymbol{S}_3$, if $\boldsymbol{S}_1$ can be linearly expressed by $\boldsymbol{S}_2$, and $\boldsymbol{S}_2$ can be linearly expressed by $\boldsymbol{S}_3$, then $\boldsymbol{S}_1$ can be linearly expressed by $\boldsymbol{S}_3$. 
\end{lemma}
\begin{proof}
    For ease of understanding, we ignore the approximate information loss, assuming that the k-rank approximation picks all vectors with nonzero singular values:
    \begin{equation}
	\begin{aligned}
            \boldsymbol{M}_1  = \boldsymbol{U}_1 \boldsymbol{\Sigma}_1 \boldsymbol{V}^T 
             =[\hat{\boldsymbol{U}}_1 \boldsymbol{U}_1^0] 
            [
            \begin{array}{cc}
            \hat{\boldsymbol{\Sigma}}_1 & 0 \\
            0 & 0
            \end{array}
            ]\boldsymbol{V}_1^T,
	\end{aligned}
    \end{equation}
    where $\hat{\boldsymbol{U}}_1 \in \mathbb{R}^{d \times k_1}$, $\hat{\boldsymbol{\Sigma}_1} \in \mathbb{R}^{k_1 \times k_1}$, $rank(\boldsymbol{M}_1)=k_1$, and $\hat{\boldsymbol{U}}_1$ is an orthonormal basis of $\boldsymbol{M}_1$, thus $\boldsymbol{M}_1$ can be linearly expressed by $\hat{\boldsymbol{U}}_1$. 
    \begin{equation}
	\begin{aligned}
            [\hat{\boldsymbol{U}}_1 \boldsymbol{M}_2]  = \boldsymbol{U}_2 \boldsymbol{\Sigma}_2 \boldsymbol{V}_2^T 
             =[\hat{\boldsymbol{U}}_2 \boldsymbol{U}_2^0] 
            [
            \begin{array}{cc}
            \hat{\boldsymbol{\Sigma}}_2 & 0 \\
            0 & 0
            \end{array}
            ]\boldsymbol{V}_2^T,
	\end{aligned}
    \end{equation}
    where $\hat{\boldsymbol{U}}_2 \in \mathbb{R}^{d \times k_2}$, $\hat{\boldsymbol{\Sigma}_2} \in \mathbb{R}^{k_1 \times k_1}$, $rank([\hat{\boldsymbol{U}}_1; \boldsymbol{M}_2])=k_2$, similarly, $\hat{\boldsymbol{U}}_1$ and $\boldsymbol{M}_2$ can be linearly expressed by $\hat{\boldsymbol{U}}_2$. According to \textbf{Lemma \ref{lemma}}, $\boldsymbol{M}_1$ can also be linearly expressed by $\hat{\boldsymbol{U}}_2$. 
    
    By the same reasoning, it can be proved that $\boldsymbol{M}=[\boldsymbol{M}_1, \boldsymbol{M}_2, \dots, \boldsymbol{M}_n]$ can be linearly expressed by $\hat{\boldsymbol{U}}_n$, thus $\hat{\boldsymbol{U}}_n$ is a significant representation of $\boldsymbol{M}$.
\end{proof}

\noindent \textbf{Memory saving rate.} Theoretically, the memory saving rate of iSVD over SVD is about $\frac{n^2-1}{n^2}$ (see supplementary). Considering memory sharing and other factors, there are some changes in practice. Table \ref{memory} shows the actual memory usage of SVD and iSVD for computing the significant representation of ten images. When $n=10$ the memory is saved by $89.4\%$. In this paper, we set $n$ to be the image number of the previous task. Therefore, the memory consumption of iSVD does not expand as the number of images increases.

\begin{table}[h!]
\centering
\vspace{-5pt}
\scalebox{0.85}{
\setlength{\tabcolsep}{2.8mm}
\begin{tabular}{ccccc}
\toprule
 & \multicolumn{1}{c}{\textbf{SVD}} & \multicolumn{3}{c}{\textbf{iSVD}} \\ \midrule
{$n$} & 1 & 2 & 5 & 10 \\
{Memory} & 157.3GB & 69.9GB & 30.3GB & \textbf{16.7GB} \\
{Saving rate} & 0\% & 55.6\% & 80.7\% & 89.4\% \\ \bottomrule
\end{tabular}}
\vspace{-5pt}
\caption{Memory consumption and memory saving rate of iSVD with different number of split blocks $n$.}
  \label{memory}
  \vspace{-10pt}
\end{table}


\subsection{Anomaly-Masked Network}
\label{amn}
The image-to-image diffusion models in anomaly detection face the problems of "identical shortcuts" and ``overfitting''. To address the above two issues, we propose the anomaly-masked conditional network (AMN) to control anomaly-to-normal generation. As shown in Figure \ref{method}, AMN comprises local and non-local encoders. The local encoder is CNNs with SiLU activation, which is used to encode the local features. Given an input image  $x\in \mathbb{R}^{H\times W \times 3}$, $H $ and $ W$ denote its height and width, the image is initially encoded by local encoder $\bar{c}=E_{lo}(x)$, where $\bar{c} \in \mathbb{R}^{\frac{H}{f}\times \frac{W}{f} \times D}$, which downsamples by a factor of $f$ and the shape of $\bar{c}$ is consistent with the input shape of U-Net. The non-local encoder is composed of the transformer encoder layer and is used to mask out abnormal regions through the perception of non-local features. The self-attention module is replaced by neighbor-masked self-attention (NMSA) \cite{uniad}, which aims to enhance the role of non-local features in feature encoding. Concretely, the locally encoded feature map $\bar{c}$ is added with the sine positional embedding $e_p$ and then processed through the non-local encoder $c=E_{no}(\bar{c}+e_p)$, where $c \in \mathbb{R}^{\frac{H}{f}\times \frac{W}{f} \times D}$ is the reprocessed $\bar{c}$. Patch perturbation is introduced to make the non-local encoder mask out abnormal features and retain normal features. We perturb the original image $x$ to obtain abnormal images $\Tilde{x}$, both of which are input to AMN and constrained by anomaly-masked loss:
\begin{equation}
    \begin{aligned}
        \mathcal{L}_{amn} = \mathbb{E}_{x, \Tilde{x}}\left\| E_{no}(E_{lo}(\Tilde{x})+e_p)-E_{lo}(x)\right\|_2^2,
    \end{aligned}
\label{eq:amn}
\end{equation} 
which can mask out the anomaly region and retain normal region features. So that the denoising process pays more attention to the reconstruction of anomaly regions.

\section{Experiments}
\label{sec:experiments}

\noindent \textbf{Dataset.} In this paper, the datasets we use are MVTec \cite{MvTec} and VisA \cite{Visa}. Both datasets contain multiple classes. In particular, MVTec contains 15 classes with $700^2-900^2$ pixels per image, and VisA contains 12 classes with roughly $1.5\textup{K}\times1\textup{K}$ pixels per image. The training set contains only normal images, while the test set contains normal images and anomaly images with image-level and pixel-level annotations. These two datasets contain intricate instances that are widely acknowledged for their real-world applicability.

\noindent \textbf{Protocol.} Referring to IUF\cite{iuf}, we conducted comparative experiments between our approach and the other methods under different continual anomaly detection settings. Specifically, in MVTec-AD \cite{MvTec}, we performed the subsequent four settings: 14 -- 1 with 1 step, 10 -- 5 with 1 step, 3 $\times$ 5 with 5 steps, and 10 -- 1 $\times$ 5 with 5 steps. In VisA \cite{Visa}, we established three settings: 11 -- 1 with 1 step, 8 -- 4 with 1 step, and 8 -- 1 $\times$ 4 with 4 steps, . In addition, we completed the cross-dataset continuous anomaly detection, including the MVTec to VisA increment and the VisA to MVTec increment.

\noindent \textbf{Evaluation metrics.} We follow the literature \cite{MvTec} in reporting the Area Under the Receiver Operation Characteristic (AUROC) for both image-level and pixel-level anomaly detection. To measure the performance of the model in continuous learning, referring to DNE \cite{dne}, we calculated the average AUROC (A-AUROC) and the forgetting measure (FM) for $N$ continual steps.

\noindent \textbf{Implementation details.} CDAD is based on the pre-trained Stable Diffusion with a pre-trained VAE, and the VAE does not need further fine-tuning. The input images have all resized of 256 × 256, and the input to the text encoder is a space character. For the Anomaly-Masked Network, the depth of both the local encoder and the non-local encoder is $8$.  In the continuous learning process, the initial task is trained for 500 epochs, and each incremental task is trained for 100 epochs. During the inference phase, the Denoising Diffusion Implicit Model \cite{ddim} (DDIM) is employed as the sampler with a default of 10 steps, and the pre-trained CNN we used to extract features is a ResNet50 pre-trained on ImageNet.
\begin{table*}[]
\centering
\scalebox{0.80}{
\setlength{\tabcolsep}{2mm}
\begin{tabular}{c|cc|cc|cc|cc}
\toprule
\multicolumn{1}{c|}{\multirow{2}{*}{Method}}             & \multicolumn{2}{c|}{\textbf{14 -- 1 with 1 Step}} & \multicolumn{2}{c|}{\textbf{10 -- 5 with 1 Step}} & \multicolumn{2}{c|}{\textbf{3 × 5 with 5 Steps}} & \multicolumn{2}{c}{\textbf{10 -- 1 × 5 with 5 Steps}} \\ \cmidrule{2-9} 
                              & \textbf{A-AUROC ($\uparrow$)}            & \textbf{FM ($\downarrow$)}                 & \textbf{A-AUROC ($\uparrow$)}              & \textbf{FM ($\downarrow$)}                 & \textbf{A-AUROC ($\uparrow$)}              & \textbf{FM ($\downarrow$)}                 & \textbf{A-AUROC ($\uparrow$)}                & \textbf{FM ($\downarrow$)}                   \\ \hline
UniAD\cite{uniad}                 & 85.7 / 89.6        & 18.3 / 13.3          & 86.7 / 91.5          & 14.9 / 10.6          & 81.3 / 88.7          & 7.4 / 10.6           & 76.6 / 82.3            & 21.1 / 17.3            \\
UniAD  + EWC\cite{ewc}  & 92.8 / 95.4        & 4.1 / 1.9            & 90.5 / 93.6          & 7.3 / 4.2            & 79.6 / 89.0          & 9.5 / 10.1           & 89.6 / 93.8            & 5.4 / 3.6              \\
UniAD  + SI\cite{si}   & 85.7 / 89.5        & 18.4 / 13.4          & 84.1 / 88.3          & 20.2 / 17.0          & 81.9 / 88.5          & 7.0 / 10.8           & 77.9 / 82.0            & 19.5 / 17.7            \\
UniAD  + MAS\cite{mas}   & 85.8 / 89.6        & 18.1 / 13.3          & 86.8 / 91.0          & 14.9 / 11.6          & 81.5 / 89.0          & 7.2 / 10.2           & 77.9 / 82.0            & 19.5 / 17.7            \\
UniAD  + LVT\cite{lvt} & 80.4 / 86.0        & 29.1 / 20.6          & 87.1 / 90.6          & 14.1 / 12.3          & 80.4 / 88.6          & 8.6 / 10.6           & 78.2 / 88.3            & 19.1 / 16.1            \\

UCAD$*$\cite{ucad}                          & {93.8} / 95.7        & 1.8 / 0.3            & 88.7 / 93.1          & 5.2 / \underline{3.7}            & \underline{84.8} / {90.1}          & {10.3} / 9.2           & {91.2} / {94.0}            & {6.3} / {3.0}              \\ 

IUF\cite{iuf}                          & \underline{96.0} / 96.3        & \underline{1.0} / 0.6            & 92.2 / \underline{94.4}          & 9.3 / 4.3            & {84.2} / \underline{91.1}          & \underline{10.0} / 8.4           & \underline{94.2} / \underline{95.1}            & \underline{3.2} / \underline{1.0}              \\ 

\midrule
DiAD$\dag$\cite{diad}                          & 93.5 / \underline{96.5}        & 1.7 / \underline{0.1}           & 91.9 / {93.9}          & 3.2 / 3.9           & {80.5} / 89.0          & {11.8} / \underline{7.2}           & 83.3 / 92.7            & 12.2 / 3.6             \\
ControlNet$\dag$\cite{controlnet}                    & 92.6 / 96.4        & 3.6 / \underline{0.1}            & \underline{92.7} / 93.8          & \underline{2.9} / 4.0            & 79.2 / {89.6}          & 13.9 / 7.3           & 82.6 / 91.8            & 13.5 / 4.5             \\ \midrule 
\rowcolor[RGB]{230,230,230}
\textbf{CDAD}$\dag$                         & \textbf{96.4} / \textbf{96.6}        & \textbf{-0.57} / \textbf{-0.04}        & \textbf{94.2} / \textbf{95.3}          & \textbf{2.05} / \textbf{2.40}          & \textbf{89.1} / \textbf{92.2}          & \textbf{3.69} / \textbf{4.65}          & \textbf{94.9} / \textbf{95.7}            & \textbf{1.01} / \textbf{0.74}            \\ \bottomrule
\end{tabular}
}
\caption{Image-level/pixel-level results of our method on MVTec under 4 continual anomaly detection settings. The best and second-best results are marked in \textbf{blod} and \underline{underline}. $\dag$ indicates DM-based methods, and $*$ indicates memory-limited.}
\vspace{-5pt}
\label{mvtec}
\end{table*}

\begin{table*}[]
\centering
\scalebox{0.75}{
\setlength{\tabcolsep}{6mm}
\begin{tabular}{c|cc|cc|cc}
\toprule
                      \multicolumn{1}{c|}{\multirow{2}{*}{Method}}      & \multicolumn{2}{c|}{\textbf{11 -- 1 with 1 Step}} & \multicolumn{2}{c|}{\textbf{8 -- 4 with 1 Step}} & \multicolumn{2}{c}{\textbf{8 -- 1 × 4 with 4 Steps}} \\ \cmidrule{2-7} 
\multirow{-2}{*}{\textbf{}} & \textbf{A-AUROC ($\uparrow$)}                 & \textbf{FM ($\downarrow$)}                     & \textbf{A-AUROC ($\uparrow$)}                 & \textbf{FM ($\downarrow$)}                      & \textbf{A-AUROC ($\uparrow$)}                    & \textbf{FM ($\downarrow$)}                       \\ \midrule
UniAD\cite{uniad}                       & 75.0 / 92.1           & 22.4 / 11.4            & 78.1 / 94.0           & 14.7 / 8.4              & 72.2 / 90.8              & 16.6 / 9.2               \\
UniAD + EWC\cite{ewc}                 & 78.7 / 95.4           & 14.9 / 4.8             & 80.5 / 95.4           & 10.0 / 5.3              & 72.3 / 92.3              & 16.5 / 7.3               \\
UniAD + SI\cite{si}                  & 78.1 / 92.0           & 16.9 / 11.5            & 69.8 / 88.5           & 9.2 / 8.3               & 69.8 / 88.5              & 19.8 / 12.0              \\
UniAD + MAS\cite{mas}                 & 75.4 / 91.8           & 21.5 / 11.9            & 72.1 / 90.6           & 14.1 / 8.4              & 72.1 / 90.6              & 16.7 / 9.4               \\
UniAD + LVT\cite{lvt}                 & 77.5 / 92.3           & 17.3 / 10.9            & 70.8 / 91.4           & 13.4 / 8.1              & 70.8 / 94.4              & 18.3 / 8.4               \\
UCAD$*$\cite{ucad}                         & 85.9 / 94.5          & 2.7 / 0.6              & 79.9 / 94.2           & 8.9 / 4.8               & 78.8 / 93.5              & {10.4} / \underline{5.4}                \\

IUF\cite{iuf}                         & \underline{87.3} / \textbf{97.6}           & \underline{2.4} / 1.8              & \underline{80.1} / \underline{95.4}           & 9.8 / 6.8               & \underline{79.8} / \underline{95.0}              & \underline{9.8} / 6.8                \\

\midrule
DiAD$\dag$\cite{diad}                        & 70.2 / 91.1           & 5.6 / 2.2              & 67.4 / 91.7           & 7.7 / 1.6               & 65.0 / 84.9              & 11.2 / 8.8               \\
ControlNet$\dag$\cite{controlnet}                  & 82.4 / 92.3           & 2.6 / \underline{0.2}              & 76.3 / 91.8           & \underline{5.8} / \underline{0.4}               & 70.0 / 88.7              & 13.8 / {5.8}               \\ \midrule
\rowcolor[RGB]{230,230,230}
\textbf{CDAD}$\dag$                        & \textbf{88.3} / \underline{97.2}           & \textbf{-1.28} / \textbf{-0.06}          & \textbf{85.3} / \textbf{97.1}           & \textbf{1.39} / \textbf{0.29}           & \textbf{83.4} / \textbf{95.8}              & \textbf{3.79} / \textbf{1.49}              \\ \bottomrule
\end{tabular}}
\caption{Image-level/pixel-level results of our method on VisA under 3 continual anomaly detection settings. The best and second-best results are marked in \textbf{blod} and \underline{underline}. $\dag$ indicates DM-based methods, and $*$ indicates memory-limited.}
\label{visa}
\vspace{-10pt}
\end{table*}

\begin{table}[] 
\centering

\scalebox{0.8}{
\setlength{\tabcolsep}{0.6mm}
\begin{tabular}{c|cc|cc}
\toprule
\multicolumn{1}{c|}{\multirow{2}{*}{Method}}               & \multicolumn{2}{c|}{\textbf{MVTec--VisA}} & \multicolumn{2}{c}{\textbf{VisA--MVTec}} \\ \cmidrule{2-5} 
\multicolumn{1}{l|}{\multirow{-2}{*}{}} & \textbf{A-AUROC ($\uparrow$)}              & \textbf{FM ($\downarrow$)}                   & \textbf{A-AUROC ($\uparrow$)}             & \textbf{FM ($\downarrow$)}                   \\ \midrule
IUF                                     & 82.4 / 92.5            & {11.09} / {4.88}            & 74.3 / 89.7           & 16.74 / \underline{5.54}           \\
DiAD                                    & \underline{85.5} / \underline{93.0}            & \underline{10.24} / \underline{4.15}           & 78.1 / \underline{92.2}           & 12.54 / 5.89           \\
ControlNet                              & 85.3 / 91.7            & 10.35 / 4.73           & \underline{78.9} / 86.2           & \underline{12.07} / 10.68          \\ \midrule

\rowcolor[RGB]{230,230,230}
\textbf{CDAD}                                    & \textbf{90.0} / \textbf{94.9}            & \textbf{4.71} /  \textbf{1.78}            & \textbf{84.2} /  \textbf{94.1}           & \textbf{6.90} /  \textbf{3.45 }            \\ \bottomrule
\end{tabular}}
\caption{Image-level/pixel-level results on cross-dataset continual anomaly detection.}
\label{cross-data}
\vspace{-15pt}
\end{table}
\subsection{Comparison Results}

The image/pixel-level experimental results of CDAD and current methods on continual anomaly detection are recorded in Table \ref {mvtec} \ref{visa}, where UniAD\cite{uniad} and DiAD\cite{diad} are multi-class anomaly detection methods, EWC\cite{ewc}, SI\cite{si}, MAS\cite{mas} and LVT\cite{lvt} are traditional continual learning method. IUF\cite{iuf} and UCAD\cite{ucad} are incremental frameworks for continual anomaly detection. We also replicate the results of DiAD\cite{diad} and ControlNet\cite{controlnet} on continual anomaly detection. As illustrated in Table \ref{mvtec} \ref{visa}, after incorporating traditional continuous learning methods, the results of UniAD showed some improvement but were still below the SOTA level, the DM-based methods achieve good results when the number of steps is small, while deteriorate and suffer ``catastrophic forgetting'' when the number of steps is larger. Compared with the above methods, IUF has achieved better results overall, but in some tasks, such as setting 3 on MVTec and setting 6 on visa, it still suffers from forgetting problems. Regarding our method, CDAD achieved all the best results on MVTec, especially on setting 3, which increased by $4.3 / 1.1$ compared to the second place, and the forgetting rate was also significantly reduced by $6.31 / 3.75$. On the VisA dataset, in addition to the pixel level result of setting 5, CDAD achieves the best results on other settings and improves $5.2 / 1.7$ and $3.6 / 0.8$ compared to the SOTA level of setting 6 and setting 7, respectively. The recorded results indicate that CDAD exhibits excellent anomaly detection performance and a significant anti-forgetting capability. Notably, in MVTec setting 1 and VisA setting 5, CDAD not only mitigates the issue of forgetting but also demonstrates enhanced performance on previous tasks, which suggests that CDAD can improve the overall generalization of the model with a limited number of incremental tasks.

\subsection{Cross-Dataset Continual Learning}

Table \ref{cross-data} shows the results of cross-dataset continual anomaly detection, MVTec--VisA represents base training on MVTec followed by incremental training on VisA. Due to the memory efficiency, CDAD enables continual learning on large-scale datasets. As recorded, CDAD outperforms the second-best method by $4.5 / 1.9$ and $5.3 / 1.9$ on two cross-dataset settings, respectively. In addition, CDAD also achieves the optimal cross-dataset forgetting rate, which indicates its excellent stability for large-scale incremental training.

\begin{figure*}[t]
  \centering
   \includegraphics[width=1.0\linewidth]{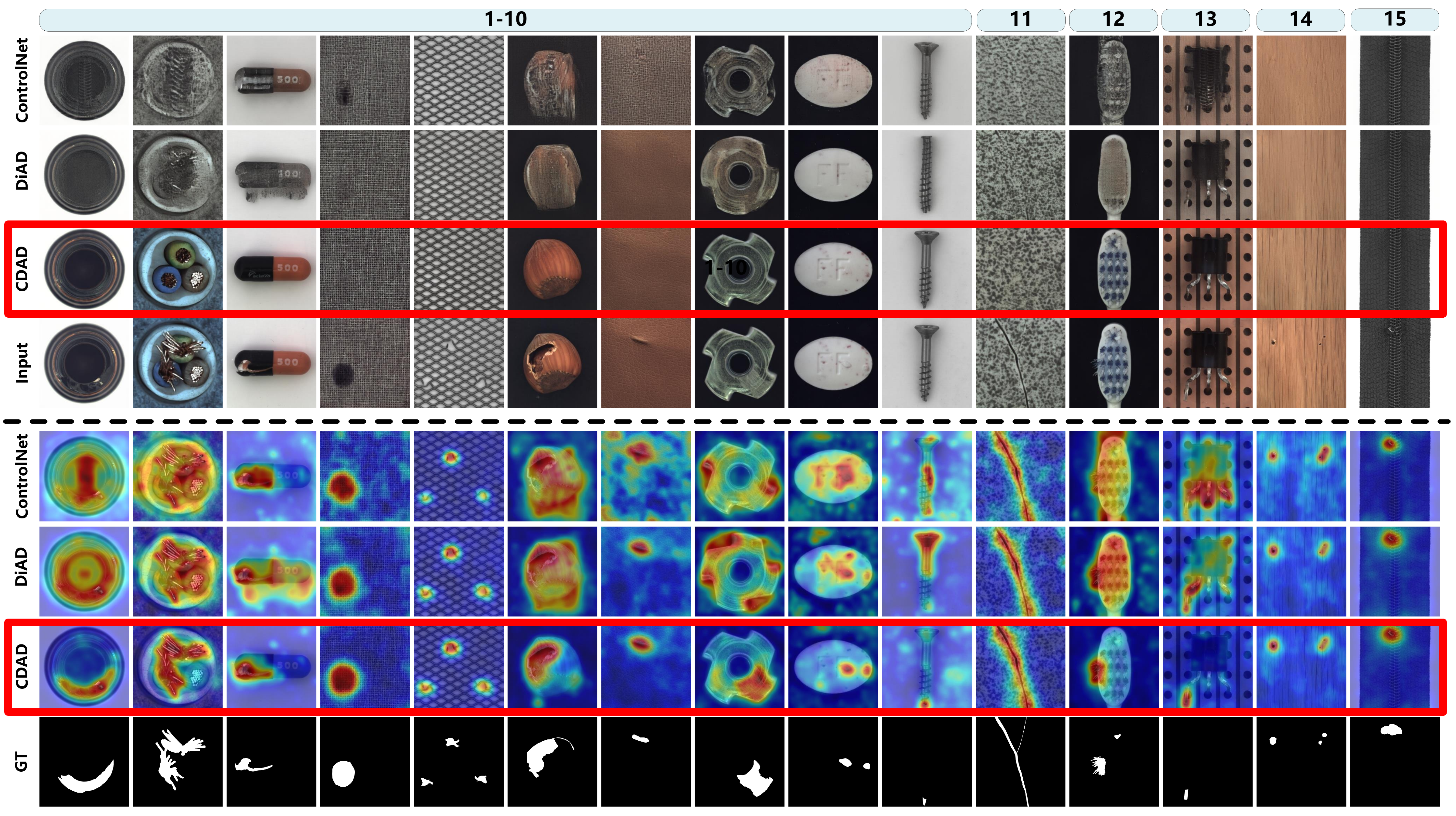}
    \vspace{-10pt}
   \caption{Qualitative comparison results under setting 4 of MVTec, the numbers represent continual training steps. The \textbf{top} is the reconstruction result of the anomaly image, and the \textbf{bottom} is the heat map of the anomaly location.}
   \vspace{-10pt}
   \label{qua}
\end{figure*}

\subsection{Visualization Results}
Figure \ref {qua} shows the qualitative results. It can be seen that on the last task, all methods can show excellent anomaly reconstruction and anomaly location effects, however, ControlNet \cite{controlnet} and DiAD \cite{diad} have serious illusion on the reconstruction results of the previous task, and the results of anomaly localization also collapse, especially on the object class. Our method shows strong stability in continuous learning, which shows excellent anomaly reconstruction and localization effects not only in the latest task but also in previous tasks. This indicates that our method effectively mitigates the hallucination problem caused by continual learning in diffusion models.
 
\subsection{Ablation Study}
We verify the impact of different modules on the overall performance of CDAD under setting 4 on MVTec \cite{MvTec}. These include the continual diffusion model (CDM) and anomaly-masked network (AMN). To better verify the effectiveness of our method, we use ControlNet \cite{controlnet} as the baseline for comparison. Results of the ablation study are recorded in Table \ref {ab}.

\noindent \textbf{Continual Diffusion Model.}  The continual diffusion model (CDM) is proposed to better retain the model's knowledge of previous tasks during continual learning, thereby alleviating the problem of forgetting. As shown in Table \ref{ab}, after removing CDM, the performance of the model in the new task gradually decreases with the increase of steps, and the final forgetting rate also increases by $4.35/1.81$. However, compared with ControlNet, the forgetting rate of the model without CMD (only use AMN) decreased by $8.1/1.9$, which is a significant improvement.

\noindent \textbf{Anomaly-Masked Network.} The anomaly-masked network (AMN) is proposed as a condition mechanism for the denoising process. In image-to-image tasks, the anomalous region in the input image is masked to retain the normal region, allowing the diffusion model to focus more effectively on reconstructing the abnormal area. As illustrated in Table \ref{ab}, without AMN, the model's results have a $1.1/1.8$ decrease on the initial step and a $0.6/1.7$ decrease after 5 steps of continuous learning. and finally. Interestingly, the use of AMN reduces the forgetting rate of the Model by $0.68/0.81$, suggesting that AMN not only enhances the performance of diffusion models in anomaly detection but also better combats forgetting.

\begin{table*}[]
\centering
\scalebox{0.8}{

\setlength{\tabcolsep}{6mm}
\begin{tabular}{ccccccc|c}
\toprule
         Step  & 1-10          & 11          & 12          & 13          & 14          & 15          & FM           \\ \midrule
ControlNet & \textbf{95.8} / \underline{97.8} & 92.4 / 95.8 & 93.2 / 96.2 & 87.4 / 91.0 & 86.9 / 91.2 & 82.6 / 91.8 & 13.46 / 4.46 \\
w/o AMN    & \underline{94.7} / 96.3 & \underline{94.9} / 96.3 & \underline{94.7} / 96.3 & \underline{93.8} / \underline{94.0} & \underline{94.0} / \underline{93.6} & \underline{94.3} / \underline{94.0} & \underline{1.69} / \underline{1.55}  \\
w/o CDM    & \textbf{95.8} / \textbf{98.1} & 93.2 / \underline{96.6} & 93.1 / \underline{96.4} & 87.9 / 92.4 & 88.3 / 91.4 & 90.6 / 93.5 & 5.36 / 2.55  \\ \midrule
CDAD       & \textbf{95.8} / \textbf{98.1} & \textbf{95.8} / \textbf{97.3} & \textbf{95.3} / \textbf{97.4} & \textbf{94.0} / \textbf{95.5}   & \textbf{94.3} / \textbf{95.1} & \textbf{94.9} / \textbf{ 95.7} & \textbf{1.01} / \textbf{0.74}  \\ \bottomrule
\end{tabular}}
\vspace{-5pt}
\caption{Ablation study results (image-level/pixel-level A-AUROC ($\uparrow$) and FM ($\downarrow$) on setting 4 of MVTec.}
\vspace{-10pt}
\label{ab}
\end{table*}

\begin{figure}[h]
  \centering
  \begin{subfigure}{0.27\linewidth}
    \includegraphics[width=1\linewidth]{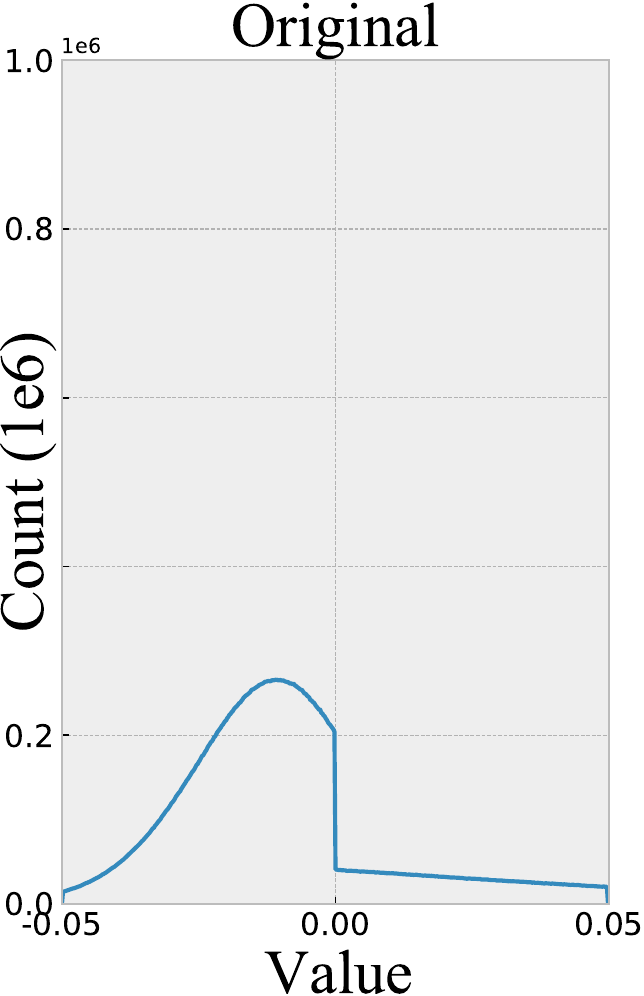}
    
  \end{subfigure} \quad
    \begin{subfigure}{0.27\linewidth}
    \includegraphics[width=1\linewidth]{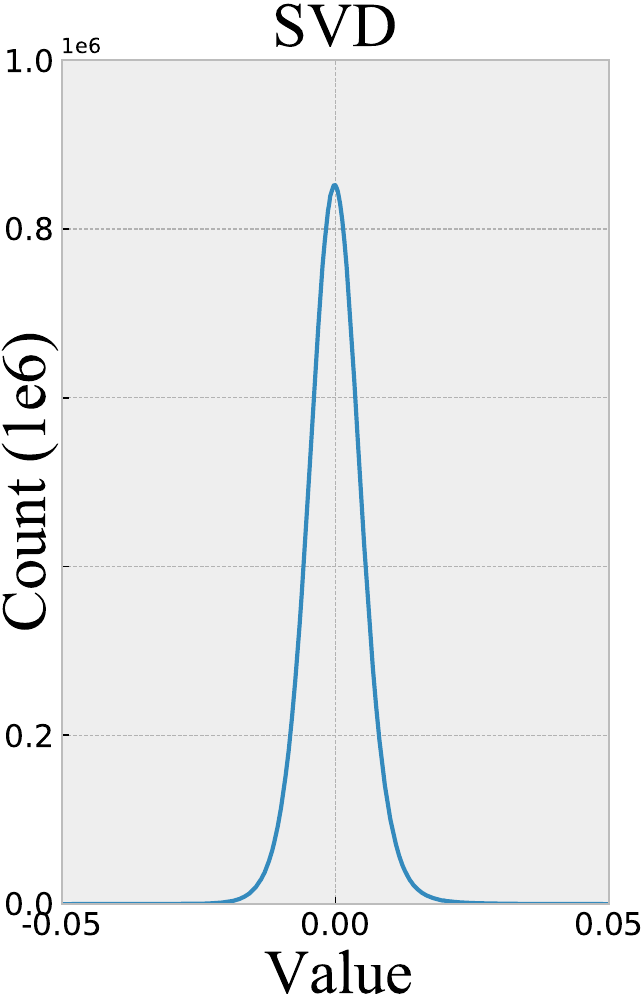}
  \end{subfigure} \quad
  \begin{subfigure}{0.27\linewidth}
    \includegraphics[width=1\linewidth]{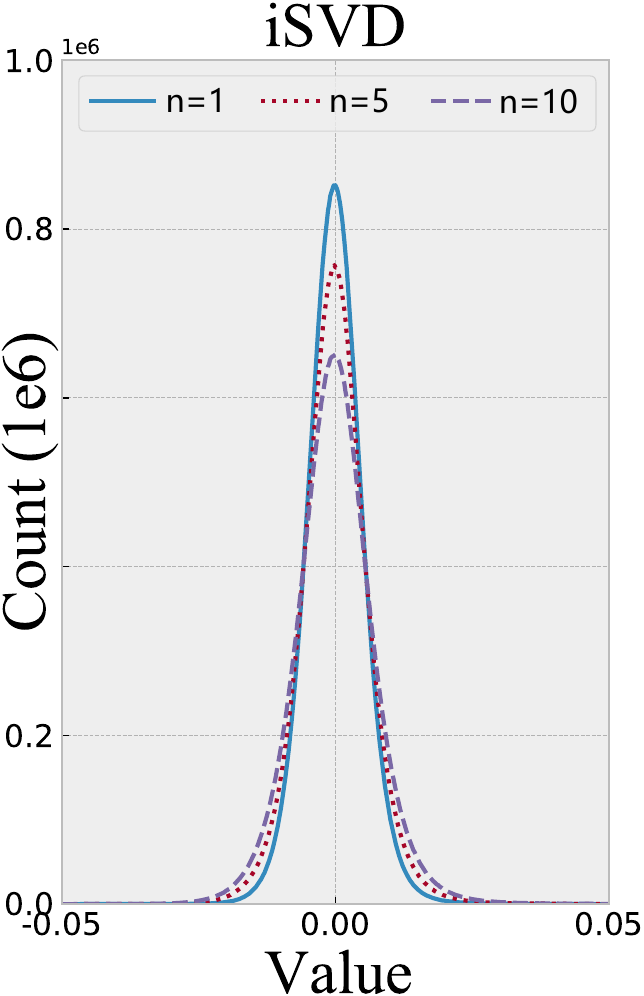}
  \end{subfigure}
  \vspace{-5pt}
  \caption{Value distribution of original matrix. Value distribution after projected by SVD and iSVD. $n$ is the number of split blocks.}
  \vspace{-15pt}
  \label{ve}
\end{figure}

\subsection{Validation Experiments of iSVD}
We complete the experiments to verify the effectiveness of iSVD. Specifically, we expand an intermediate feature map of U-Net into a matrix, using SVD and iSVD to project this matrix into its orthogonal space. Figure \ref {ve} illustrates the value distribution of the matrix before and after projection, the matrix size is $1280 \times 50000$ and the threshold for k-rank approximate is $0.98$. It can be seen that the distribution of features is relatively uniform before projection, and after SVD projection, the distribution of values is concentrated around $0$. iSVD shows almost the same performance as SVD, which verifies that the significant representation calculated by iSVD is effective. In addition, we also find that with the increase of the number of split blocks, the effect of iSVD becomes worse, this is because the k-rank approximate is performed at each iteration, which loses part of the information. However, continuous learning can not project gradients 100\% into the orthogonal space of the previous task, which will harm the model's plasticity to new tasks, so the information loss of iSVD is acceptable.

\begin{figure}
  \centering
   \includegraphics[width=0.85\linewidth]{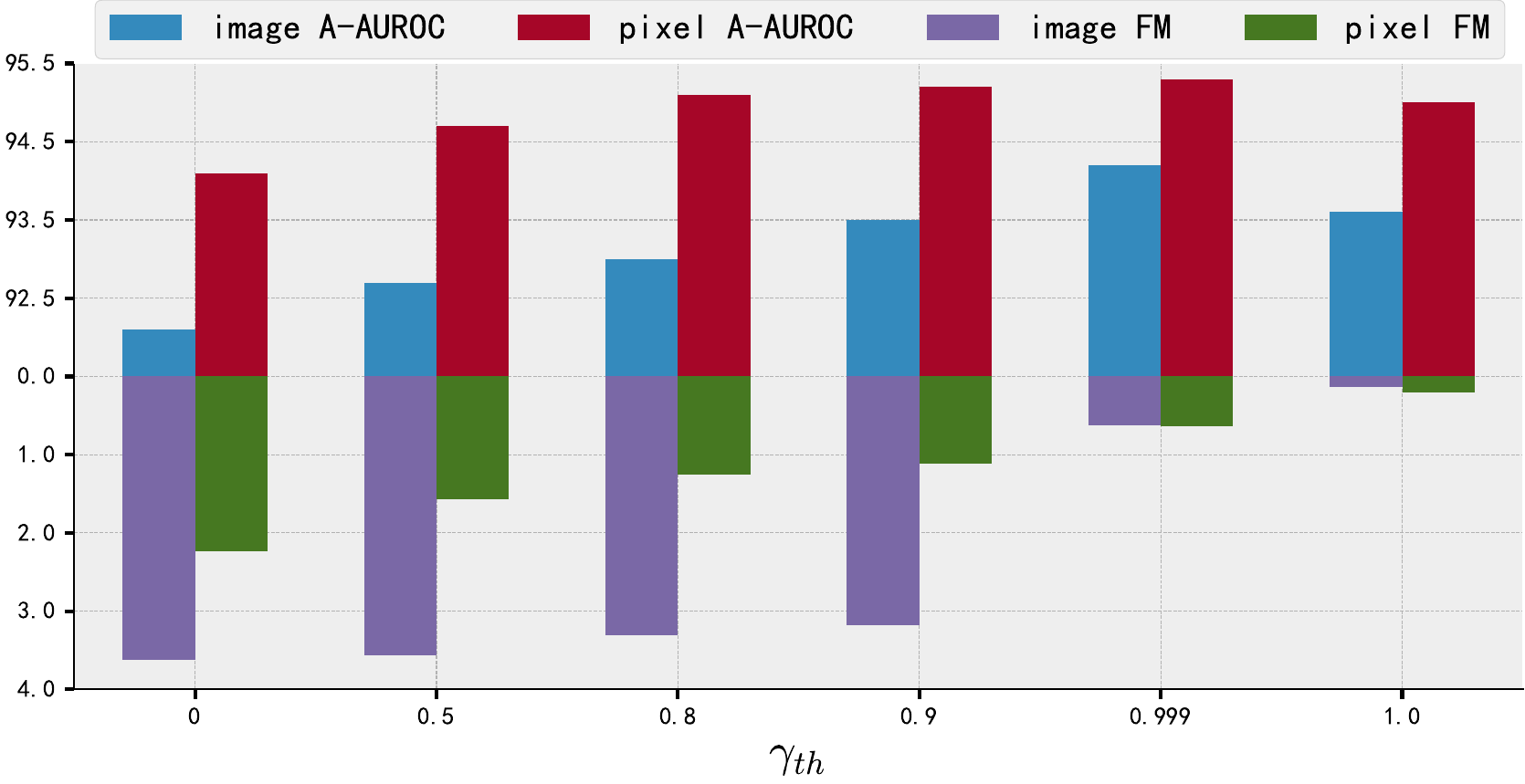}
   \vspace{-5pt}
   \caption{Impact of $\gamma_{th}$ on A-AUROC ($\uparrow$) and FM ($\downarrow$) under setting 2 of MVTec.}
   \vspace{-15pt}
  \label{gamma}
\end{figure}

\subsection{Hyper-parameter Analysis}

We complete the effect of the threshold $\gamma_{th}$ of k-rank approximation. $\gamma_{th}$ is the threshold for selecting the significant representation in the unitary matrix, used to control the amount of information retained from previous tasks. Figure \ref {ve} illustrates the effect of $\gamma_{th}$ under setting 2 of MVTec. As shown, when $\gamma_{th}=0$, the model experiences significant forgetting and performance degradation. As $\gamma_{th}$ increases, both the forgetting rate and performance improve. When $\gamma_{th}=1.0$, although the forgetting rate is minimal, performance drops due to the model retaining all column bases from previous tasks, which interferes with learning new tasks. Even when $\gamma_{th}=0.999$, the selected vectors are only a small subset of the unitary matrix, preserving the core knowledge of previous tasks while maintaining the model's plasticity to new tasks. Therefore, we set $\gamma=0.999$.

\section{Conclusion}
In this paper, we propose a continual diffusion model for anomaly detection termed CDAD, which continually learns anomaly detection for new tasks. First, we introduce a continual diffusion model framework that uses significant representation to project the training gradient of the new task to a direction having less influence on the previous tasks, to ensure the stability of the model. Second, we propose the iterative singular value decomposition, which greatly alleviates the memory pressure of computing the significant representation for diffusion models. Finally, an anomaly-masked network is proposed as a control mechanism for the diffusion model. The experimental results show that both the anomaly detection performance and anti-forgetting capability of our method are state-of-the-art.
\\

\noindent\textbf{Acknowledgments}
This work is supported by the National Natural Science Foundation of China No.62222602, 62176092, U23A20343, 62106075, 62302167, 62302296, 721928221, Shanghai Sailing Program (23YF1410500), Natural Science Foundation of Shanghai (23ZR1420400), Natural Science Foundation of Chongqing, China (CSTB2023NSCQ-JQX0007, CSTB2023NSCQ-MSX0137), CCF-Tencent RAGR20240122, and Development Project of Ministry of Industry and Information Technology (ZTZB.23-990-016).

%% file: sec/X_suppl.tex
\clearpage
\maketitlesupplementary
\renewcommand\thesection{\Alph{section}}

\setcounter{page}{1}
\setcounter{section}{0}

\section{Experimental details}
\label{sec:details}
\noindent \textbf{Data pre-processing}. We employ the data pre-processing pipeline specified in DiAD \cite{diad} for both the MVTec \cite{MvTec} and VisA \cite{Visa} datasets to mitigate potential train-test discrepancies. This involves channel-wise standardization using precomputed mean $[0.485, 0.456, 0.406]$ and standard deviation $[0.229, 0.224, 0.225]$ after normalizing each RGB image to $[0, 1]$.

\noindent \textbf{Patch perturbation}. We adopt the method proposed by NSA \cite{NSA} for patch perturbation on the original images. The NSA method builds upon the Cut-paste technique \cite{CutPaste} and enhances it by incorporating the Poisson image editing method \cite{poisson} to alleviate the discontinuities caused by pasting image patches. The cut-paste method is commonly used in the anomaly detection domain to generate simulated anomalous images. It involves randomly cropping a patch from one image and pasting it onto a random location in another image, thus creating a simulated anomaly. The Poisson-based pasting method seamlessly blends the cloned object from one image into another by solving Poisson partial differential equations, thereby better simulating a realistic anomalous region. In this paper, the number of patches is set as a random value from $1$ to $4$, and the patch size is a random value from $0.03$ to $0.4$ of the original image size. The visualization of patch perturbation is shown in Figure \ref{nsa}.

\begin{figure}[h!]
  \centering
   \includegraphics[width=1\linewidth]{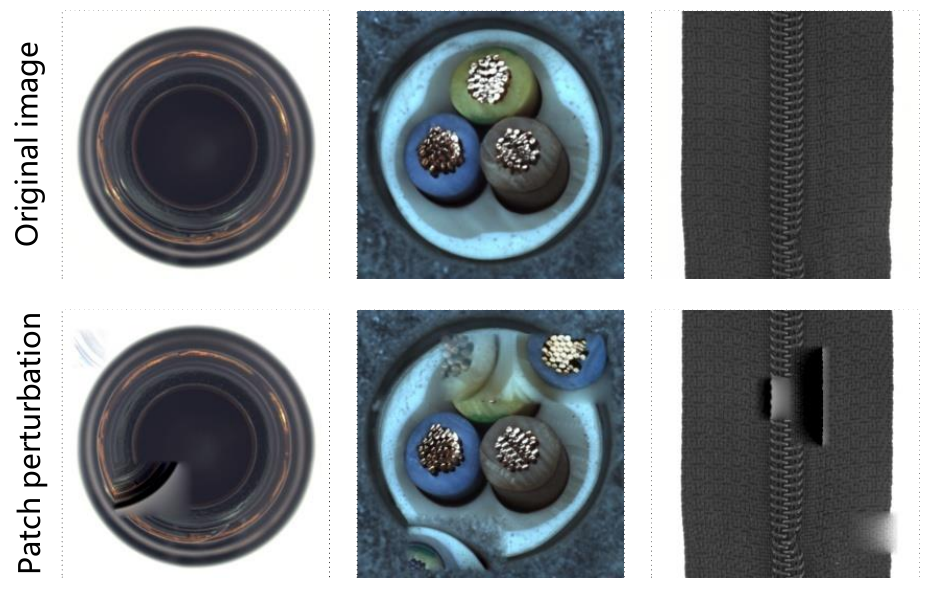}
   \caption{Qualitative results of logical anomaly detection.}
   \label{nsa}

\end{figure}
\noindent \textbf{Evaluation metrics.} We follow the literature \cite{MvTec} in reporting the Area Under the Receiver Operation Characteristic (AUROC) for both image-level and pixel-level anomaly detection. To measure the performance of the model in continuous learning, referring to DNE \cite{dne}, we calculated the average AUROC (A-AUROC) and the forgetting measure (FM) for $N$ continual steps. Specially, we define $A_{N, i}^{\mathrm{pix}}$ and $A_{N, i}^{\mathrm{img}}$ as the test AUROC of task $i$ after training on task $N$.
\begin{equation}
    \mathrm{A\textbf{-}AUROC}=\left\{\begin{array}{l}
    \frac{1}{N} \sum_{i=1}^{N-1} A_{N, i}^{\mathrm{pix}} \\
    \frac{1}{N} \sum_{i=1}^{N-1} A_{N, i}^{\mathrm{img}}
    \end{array},\right.
    \label{a-auroc}
\end{equation}
\begin{equation}
    \mathrm{FM}=\left\{\begin{array}{l}
    \frac{1}{N-1} \sum_{i=1}^{N-1} \max _{b \in\{1, \cdots, N-1\}}\left(A_{b, i}^{\mathrm{pix}}-A_{N, i}^{\mathrm{pix}}\right) \\
    \frac{1}{N-1} \sum_{i=1}^{N-1} \max _{b \in\{1, \cdots, N-1\}}\left(A_{b, i}^{\mathrm{img}}-A_{N, i}^{\mathrm{img}}\right)
    \end{array}\right.
    \label{fm}
\end{equation}
In addition to the results for the AUROC documented in the body of the paper, We also supplement the image-level Precision-Recall (AUPR) results and pixel-level Per-region-overlap (PRO) \cite{MvTec, UniS-T} results. Referring to Equation (\ref{a-auroc}) and Equation (\ref{fm}), we calculate A-AUPR, A-PRO, and their FM to evaluate our method. The results are shown in Table \ref{mvtec_aupr} \ref{mvtec_pro} \ref{visa_aupr} \ref{visa_pro}. Our method still achieves an advanced level in the above metrics.
\section{Memory Analysis of iSVD}
\label{sec:details}
In Section 3.2, considering the storage of the original matrix and $\boldsymbol{U}, \boldsymbol{\Sigma}, \boldsymbol{V}$ during SVD, the memory overhead of SVD is $d\Lambda + d^2+\Lambda^2+\min(d, \Lambda)$, while iSVD uses a memory overhead of $d(m+k) + d^2 + (m+k)^2 + \min(d, m+k)$. It is known that $\Lambda \gg d$, $m > d > k$ and $\Lambda=mn$,  thus, the memory saving rate of iSVD is about:
\begin{equation}
    \begin{aligned}
       & \frac{d\Lambda + d^2+\Lambda^2+d - [d(m+k) + d^2 + (m+k)^2 + d]}{d\Lambda + d^2+\Lambda^2+d} \\
      = & \frac{d\Lambda +\Lambda^2 - d(m+k) - (m+k)^2 }{\Lambda^2}/\frac{d\Lambda + d^2+\Lambda^2+d}{\Lambda^2} \\
      = & \frac{d}{\Lambda} + 1 - \frac{d(m+k)}{\Lambda^2} - \frac{(m+k)^2}{\Lambda^2} \\
      \approx & 1-\frac{m^2+2mk+k^2}{m^2n^2} \\
      = & 1 - \frac{1}{n^2} - \frac{2k}{m\Lambda} - \frac{k^2}{\Lambda^2}
      \approx \frac{n^2-1}{n^2}.
    \end{aligned}
\label{eq:amn}
\end{equation} 
In practice, the actual memory saving rate differs from the theoretical value due to factors such as memory sharing. Taking the intermediate features of ten images as an example, Figure \ref{numblocks} shows the actual and theoretical memory saving rate of splitting the feature matrix into $n$ blocks for iSVD. Although there are some differences between the theoretical value and the actual value, the general trend is consistent.

\begin{table*}[]
\centering
\scalebox{0.94}{
\setlength{\tabcolsep}{1.8mm}
\begin{tabular}{c|cc|cc|cc|cc}
\toprule
\multicolumn{1}{c|}{\multirow{2}{*}{Method}}             & \multicolumn{2}{c|}{\textbf{14 -- 1 with 1 Step}} & \multicolumn{2}{c|}{\textbf{10 -- 5 with 1 Step}} & \multicolumn{2}{c|}{\textbf{3 × 5 with 5 Steps}} & \multicolumn{2}{c}{\textbf{10 -- 1 × 5 with 5 Steps}} \\ \cmidrule{2-9} 
                              & \textbf{A-AUPR ($\uparrow$)}            & \textbf{FM ($\downarrow$)}                 & \textbf{A-AUPR ($\uparrow$)}              & \textbf{FM ($\downarrow$)}                 & \textbf{A-AUPR ($\uparrow$)}              & \textbf{FM ($\downarrow$)}                 & \textbf{A-AUPR ($\uparrow$)}                & \textbf{FM ($\downarrow$)}                   \\ \midrule

UCAD* \cite{ucad}                    & 95.8          & 0.26          & 95.0            & {\ul 0.98}    & {\ul 93.1}      & {\ul 2.02}    & \underline{95.5} & {\ul 0.07}    \\
IUF \cite{iuf}                     & {\ul 97.8}    & {\ul 0.25}    & 95.4          & 1.92          & 91.1          & 2.86          & 95.3          & 0.16          \\
ControlNet \cite{controlnet}               & 97.2          & 1.55          & {\ul 96.7}    & 1.76          & 86.7          & 6.40           & 89.0            & 7.43          \\
DiAD \cite{diad}                     & 97.4          & 0.71          & 96.4          & 1.85          & 89.1          & 4.31          & 91.4          & 4.83          \\ \midrule

\rowcolor[RGB]{230,230,230}
\textbf{CDAD}                     & \textbf{98.4} & \textbf{0.08} & \textbf{98.3} & \textbf{0.55} & \textbf{95.8} & \textbf{1.88} & \textbf{98.4} & \textbf{0.02}          \\ \bottomrule
\end{tabular}
}
\caption{Image-level A-AUPR of our method on MVTec under 4 continual anomaly detection settings. The best and second-best results are marked in \textbf{blod} and \underline{underline}. $*$ indicates memory limited.}
\label{mvtec_aupr}
\end{table*}

\begin{table*}[]
\centering
\scalebox{0.94}{
\setlength{\tabcolsep}{2.2mm}
\begin{tabular}{c|cc|cc|cc|cc}
\toprule
\multicolumn{1}{c|}{\multirow{2}{*}{Method}}             & \multicolumn{2}{c|}{\textbf{14 -- 1 with 1 Step}} & \multicolumn{2}{c|}{\textbf{10 -- 5 with 1 Step}} & \multicolumn{2}{c|}{\textbf{3 × 5 with 5 Steps}} & \multicolumn{2}{c}{\textbf{10 -- 1 × 5 with 5 Steps}} \\ \cmidrule{2-9} 
                              & \textbf{A-PRO ($\uparrow$)}            & \textbf{FM ($\downarrow$)}                 & \textbf{A-PRO ($\uparrow$)}              & \textbf{FM ($\downarrow$)}                 & \textbf{A-PRO ($\uparrow$)}              & \textbf{FM ($\downarrow$)}                 & \textbf{A-PRO ($\uparrow$)}                & \textbf{FM ($\downarrow$)}                   \\ \midrule

UCAD \cite{ucad}                    & 86.3          & 1.16          & 80.7          & {\ul 2.89}    & { 71.1}    & { 7.48}          & {80.8} & {\ul 1.19}    \\
IUF \cite{iuf}                     & { 88.6}    & {\ul 0.62}    & 85.0            & 3.22          & \underline{72.9} & {\ul 5.79}          & {\ul 84.3}    & 2.41          \\
ControlNet \cite{controlnet}               & 88.5          & 1.75          & { 85.8}    & 4.70           & 71.5          & 10.0                  & 77.9          & 7.11          \\
DiAD \cite{diad}                    & {\ul 88.9}    & 0.85          & {\ul 87.4}    & 3.91          & 72.1          & 9.23                & 83.1          & 2.93          \\ \midrule

\rowcolor[RGB]{230,230,230}
\textbf{CDAD}                     & \textbf{89.8} & \textbf{0.47} & \textbf{88.9} & \textbf{2.57} & \textbf{83.8} & { \textbf{4.05}} & \textbf{89.2} & \textbf{1.16}            \\ \bottomrule
\end{tabular}
}
\caption{Pixel-level A-PRO of our method on MVTec under 4 continual anomaly detection settings. The best and second-best results are marked in \textbf{blod} and \underline{underline}. $*$ indicates memory limited.}
\vspace{-10pt}
\label{mvtec_pro}
\end{table*}

\begin{figure}[]
  \centering
   \includegraphics[width=1\linewidth]{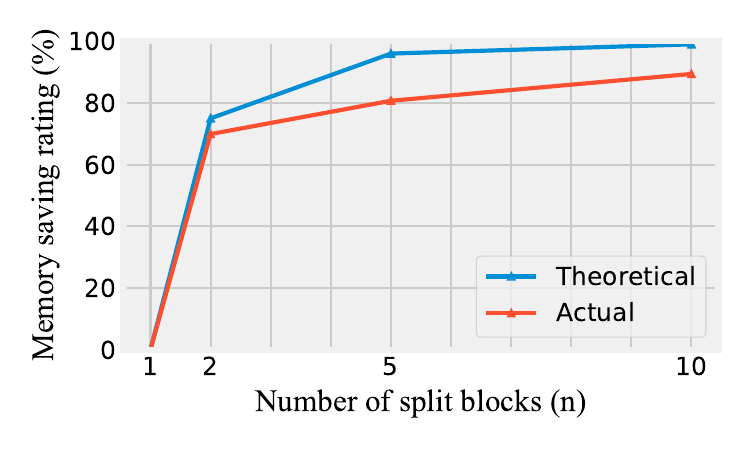}
   \caption{The theoretical and actual values of the memory saving ratio when the different number of split blocks is used.}
   \vspace{-10pt}
   \label{numblocks}
\end{figure}

In this paper, we determine the number of blocks $n$ according to the number of images of the old task. Specifically, we first sample the old task dataset, randomly retain 10\% of the images, and then group according to the number of sampled images (denoted as $N_{img}$). The number of groups in iSVD is $n=\frac{N_{img}}{e}$. In this paper, $e$ is set to $1$ by default, that is, the intermediate features of each image are divided into a separate group for iSVD operation. 

\begin{table}[h]
\begin{tabular}{c|cc|cc}
\toprule
    & \textbf{A-AUROC}     & \textbf{FM}         & \textbf{Memory} & \textbf{Times} \\ \midrule
e=1 & 94.2 / 95.3 & 2.05 / 2.40  & 16.7GB & 33.5h \\
e=2 & 94.1 / 95.4 & 2.10 / 2.32 & 30.3GB & 28.3h \\
e=4 & 94.4 / 95.6 & 1.95 / 2.21 & 57.9GB & 22.6h \\
e=6 & 94.5 / 95.7 & 1.91 / 2.23 & 85.9GB & 20.2h \\ \bottomrule
\end{tabular}
   \caption{The impact of different $e$ on the model and the time and memory overhead of iSVD under MVTec Setting 2.}
   \vspace{-10pt}
   \label{mem-times}
\end{table}

In addition, we analyze the impact of different $e$ on the model and the time and memory overhead. Table \ref{mem-times} records, for setting different $e$, the anomaly detection results of our model on MVTec setting 2, as well as the time and memory overhead for computing the significant representation of the old task. When $e$ is set to different values, the anomaly detection results and forgetting rate of the model will not have much influence. Although we discussed in Section 4.5 that the large number of split blocks will affect the performance of iSVD, it will not impair its representation ability of core information, so it can still ensure the continuous learning ability of the model. Table \ref{mem-times} also shows that with the increase of $e$, the memory consumption increases, but the time cost decreases, which indicates that although iSVD can greatly alleviate the pressure of memory, it will bring extra time cost.

\section{Qualitative Results}
\label{sec:details-quali}
We supplement the qualitative results on MVTec and VisA datasets, which show the localization image reconstruction results and anomaly localization results for the seven tasks, respectively, as shown in Figure \ref{qua1}-\ref{qua7}. Our method not only overcomes the ``faithfulness hallucination'' problem of the diffusion model but also shows excellent anomaly localization results.

\begin{table*}[]
\centering
\scalebox{0.9}{
\setlength{\tabcolsep}{5mm}
\begin{tabular}{c|cc|cc|cc}
\toprule
                      \multicolumn{1}{c|}{\multirow{2}{*}{Method}}      & \multicolumn{2}{c|}{\textbf{11 -- 1 with 1 Step}} & \multicolumn{2}{c|}{\textbf{8 -- 4 with 1 Step}} & \multicolumn{2}{c}{\textbf{8 -- 1 × 4 with 4 Steps}} \\ \cmidrule{2-7} 
\multirow{-2}{*}{\textbf{}} & \textbf{A-AUROC ($\uparrow$)}                 & \textbf{FM ($\downarrow$)}                     & \textbf{A-AUROC ($\uparrow$)}                 & \textbf{FM ($\downarrow$)}                      & \textbf{A-AUROC ($\uparrow$)}                    & \textbf{FM ($\downarrow$)}                       \\ \midrule
UCAD* \cite{ucad}                     & { 88.1}                 & 0.29                   & { 83.2}                 & {\ul 5.17}             & { 82.9}                 & {\ul 2.16}             \\
IUF \cite{iuf}                      & \textbf{91.6}              & {\ul -0.04}            & {\ul 83.4}                 & 7.51                   & {\ul 83.0}                   & 6.87                   \\
ControlNet \cite{controlnet}              & 85.2                       & 2.38                   & 78.8                       & { 6.25}             & 72.4                       & 4.56                   \\
DiAD \cite{diad}                    & 74.9                       & 5.42                   & 70.1                       & 12.29                  & 59.5                       & 9.55                   \\ \midrule
\rowcolor[RGB]{230,230,230}
\textbf{CDAD}                     & {\ul {89.4}}        & \textbf{-0.77}         & \textbf{85.3}              & \textbf{3.1}           & \textbf{84.7}              & \textbf{1.83}          \\ \bottomrule
\end{tabular}
}
\caption{Image-level A-AUPR of our method on VisA under 3 continual anomaly detection settings. The best and second-best results are marked in \textbf{blod} and \underline{underline}. $*$ indicates memory limited.}
\label{visa_aupr}
\end{table*}

\begin{table*}[]
\centering
\scalebox{0.9}{
\setlength{\tabcolsep}{6mm}
\begin{tabular}{c|cc|cc|cc}
\toprule
                      \multicolumn{1}{c|}{\multirow{2}{*}{Method}}      & \multicolumn{2}{c|}{\textbf{11 -- 1 with 1 Step}} & \multicolumn{2}{c|}{\textbf{8 -- 4 with 1 Step}} & \multicolumn{2}{c}{\textbf{8 -- 1 × 4 with 4 Steps}} \\ \cmidrule{2-7} 
\multirow{-2}{*}{\textbf{}} & \textbf{A-PRO ($\uparrow$)}                 & \textbf{FM ($\downarrow$)}                     & \textbf{A-PRO ($\uparrow$)}                 & \textbf{FM ($\downarrow$)}                      & \textbf{A-PRO ($\uparrow$)}                    & \textbf{FM ($\downarrow$)}                       \\ \midrule
UCAD* \cite{ucad}                     & {\ul 80.4}                & 2.02                   & {\ul 72.4}                & 7.46                   & {\ul 70.5}                & {\ul 9.83}             \\
IUF \cite{iuf}                     & 82.0                        & {\ul 1.04}             & 63.9                      & 20.8                   & 57.0                        & 23.95                  \\
ControlNet \cite{controlnet}               & 62.3                      & 2.45                   & 61.0                        & {\ul 1.81}             & 51.7                      & 10.38                  \\
DiAD \cite{diad}                    & 69.9                      & 4.33                   & 67.7                      & 8.29                   & 55.0                        & 11.74                  \\ \midrule

\rowcolor[RGB]{230,230,230}
\textbf{CDAD}                     & \textbf{81.6}             & \textbf{-0.22}         & \textbf{78.9}             & \textbf{1.59}          & \textbf{77.7}             & \textbf{1.66}         \\ \bottomrule
\end{tabular}
}
\caption{Pixel-level A-PRO of our method on VisA under 3 continual anomaly detection settings. The best and second-best results are marked in \textbf{blod} and \underline{underline}. $*$ indicates memory limited.}
\label{visa_pro}
\end{table*}

\begin{figure*}[h!]
  \centering
   \includegraphics[width=1.0\linewidth]{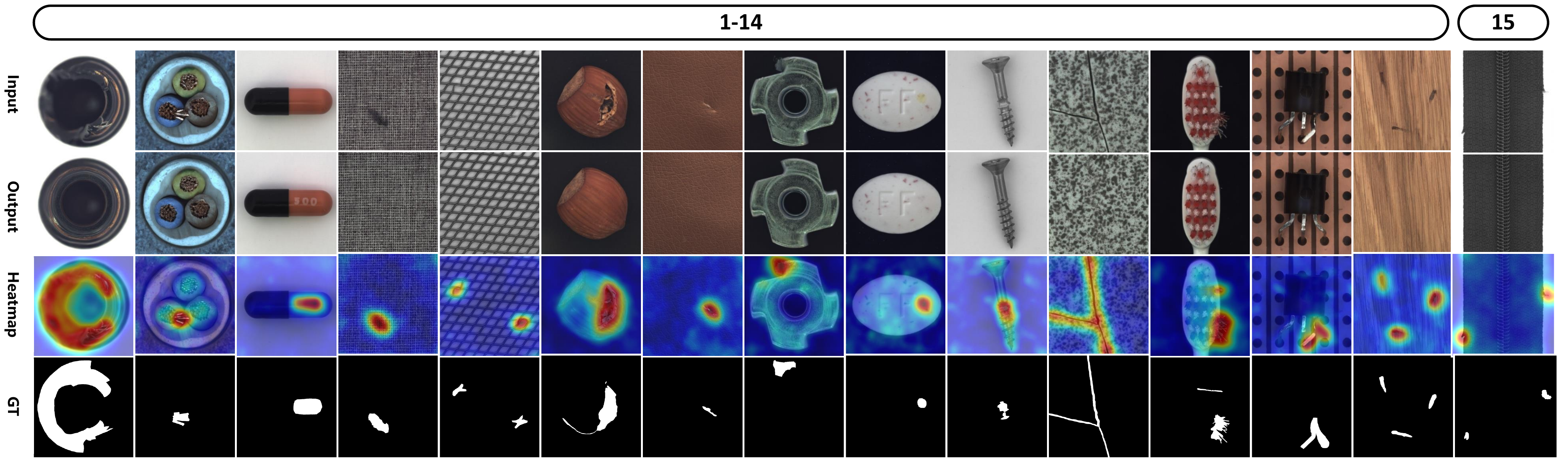}
    
   \caption{Qualitative comparison results under setting 1 of MVTec, the numbers represent continual training classes.}
   \label{qua1}
\end{figure*}

\begin{figure*}[h!]
  \centering
   \includegraphics[width=1.0\linewidth]{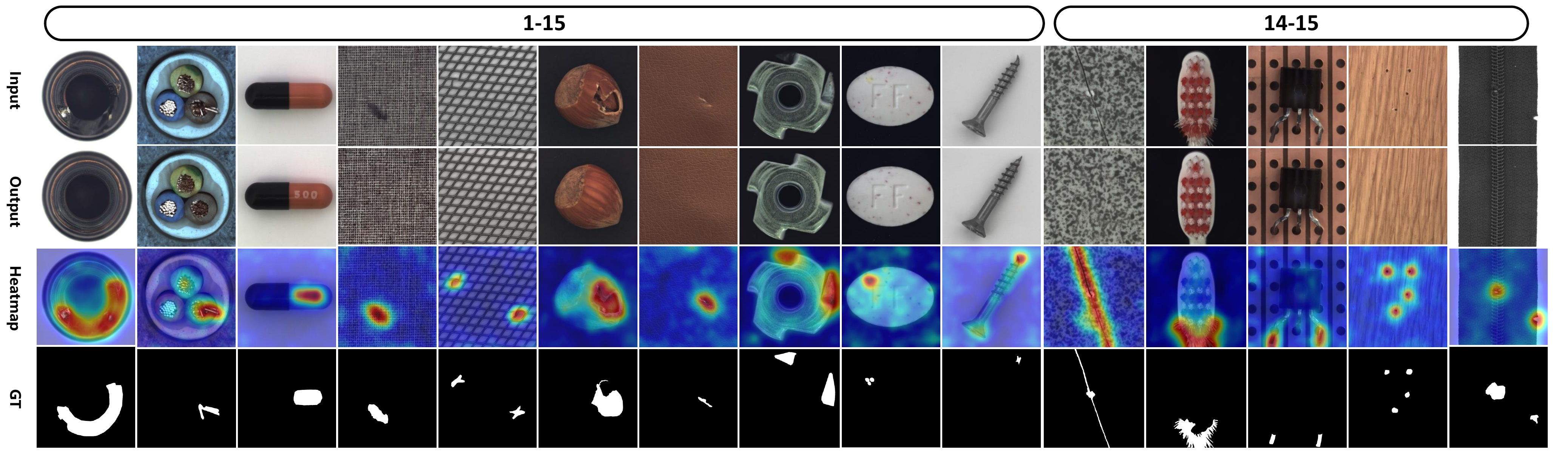}
    
   \caption{Qualitative comparison results under setting 2 of MVTec, the numbers represent continual training classes.}
   \label{qua2}
\end{figure*}

\begin{figure*}[h!]
  \centering
   \includegraphics[width=1.0\linewidth]{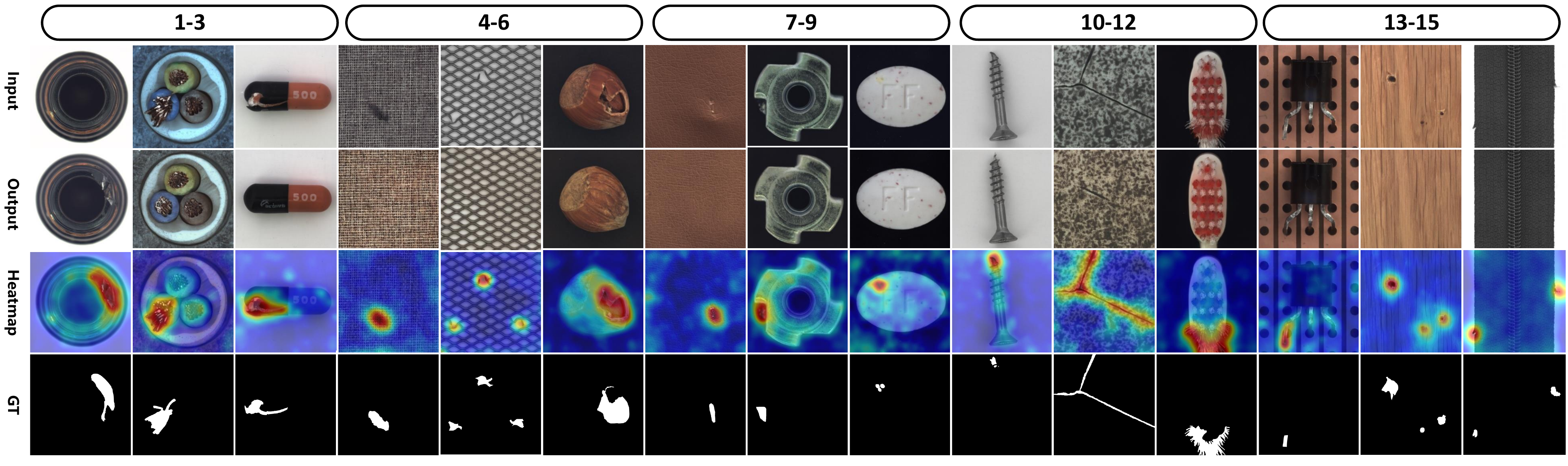}
    
   \caption{Qualitative comparison results under setting 3 of MVTec, the numbers represent continual training classes.}
   \label{qua3}
\end{figure*}

\begin{figure*}[h!]
  \centering
   \includegraphics[width=1.0\linewidth]{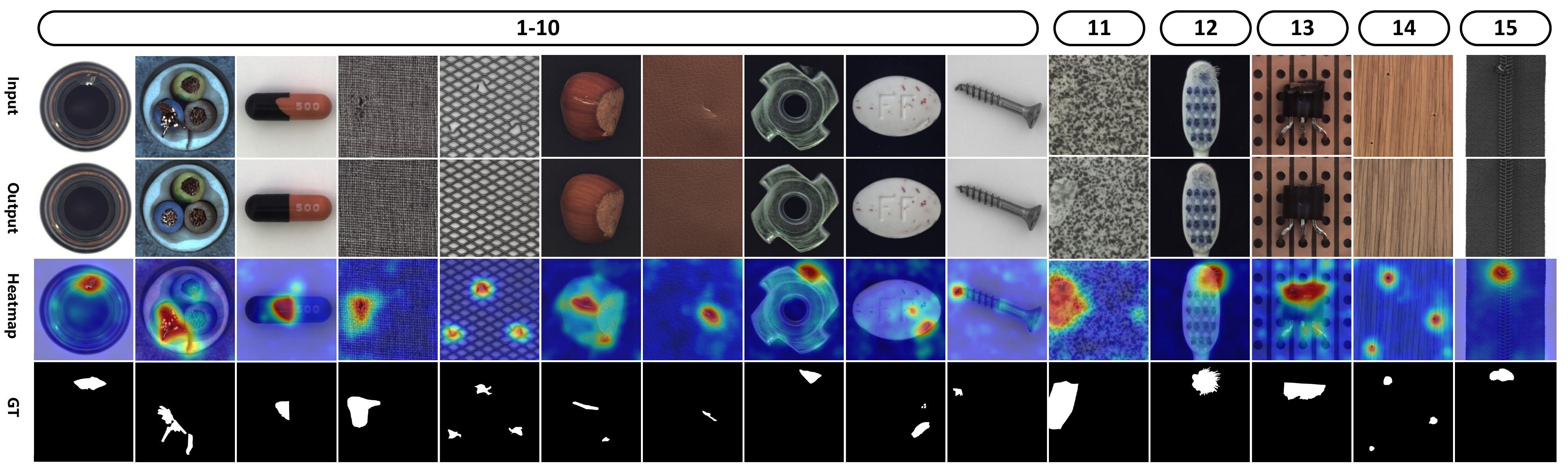}
    
   \caption{Qualitative comparison results under setting 4 of MVTec, the numbers represent continual training classes.}
   \label{qua4}
\end{figure*}

\begin{figure*}[h!]
  \centering
   \includegraphics[width=1.0\linewidth]{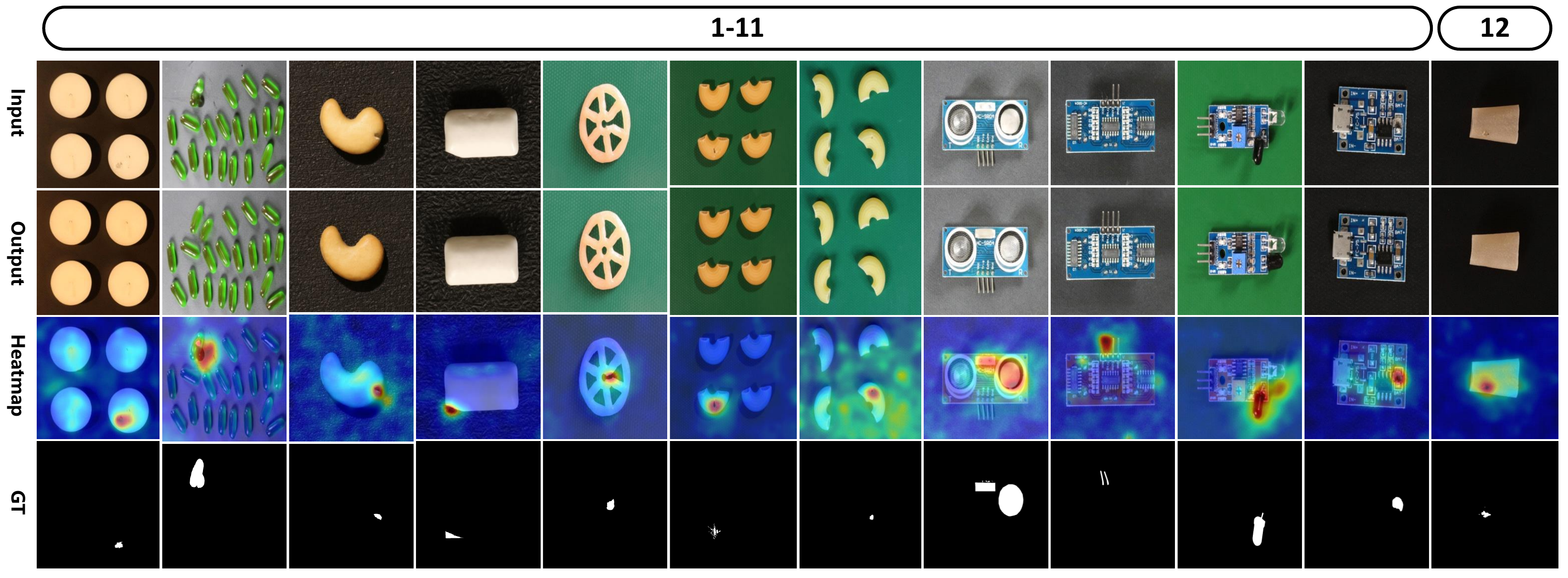}
    
   \caption{Qualitative comparison results under setting 5 of VisA, the numbers represent continual training classes.}
   \vspace{50pt}
   \label{qua5}
\end{figure*}

\begin{figure*}[h!]
  \centering
   \includegraphics[width=1.0\linewidth]{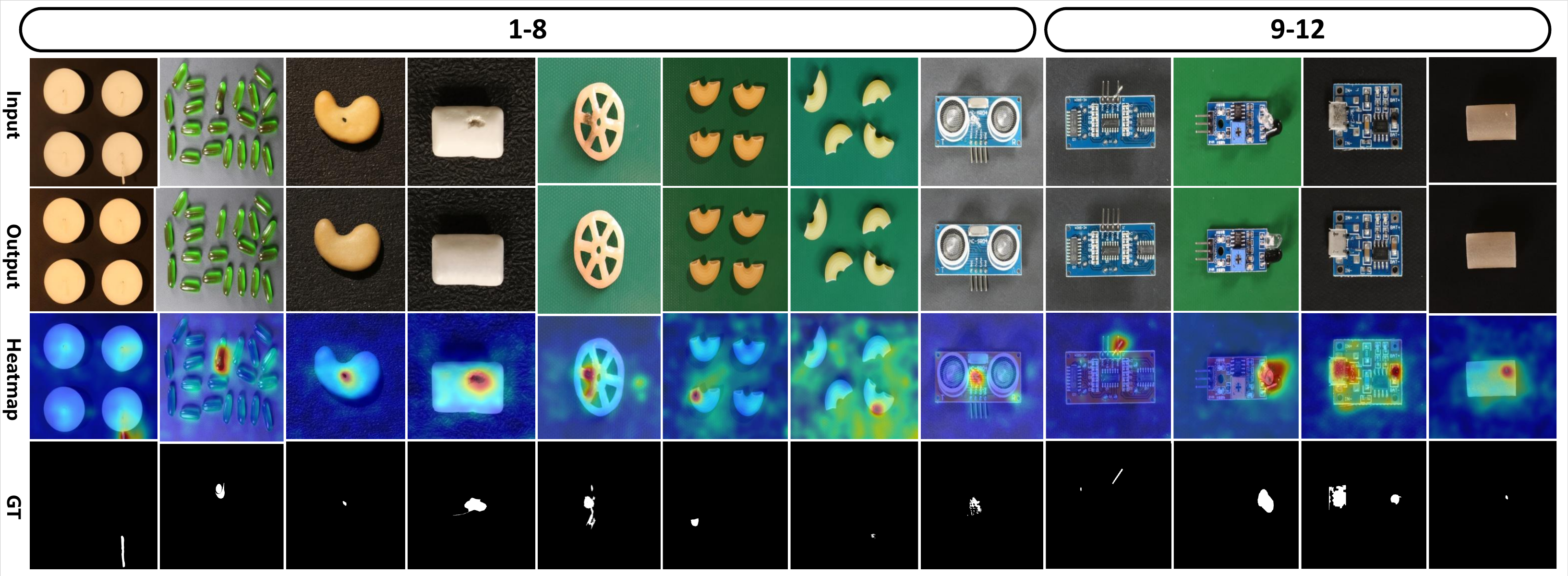}
    
   \caption{Qualitative comparison results under setting 6 of VisA, the numbers represent continual training classes.}
   \label{qua6}
\end{figure*}

\begin{figure*}[h!]
  \centering
   \includegraphics[width=1.0\linewidth]{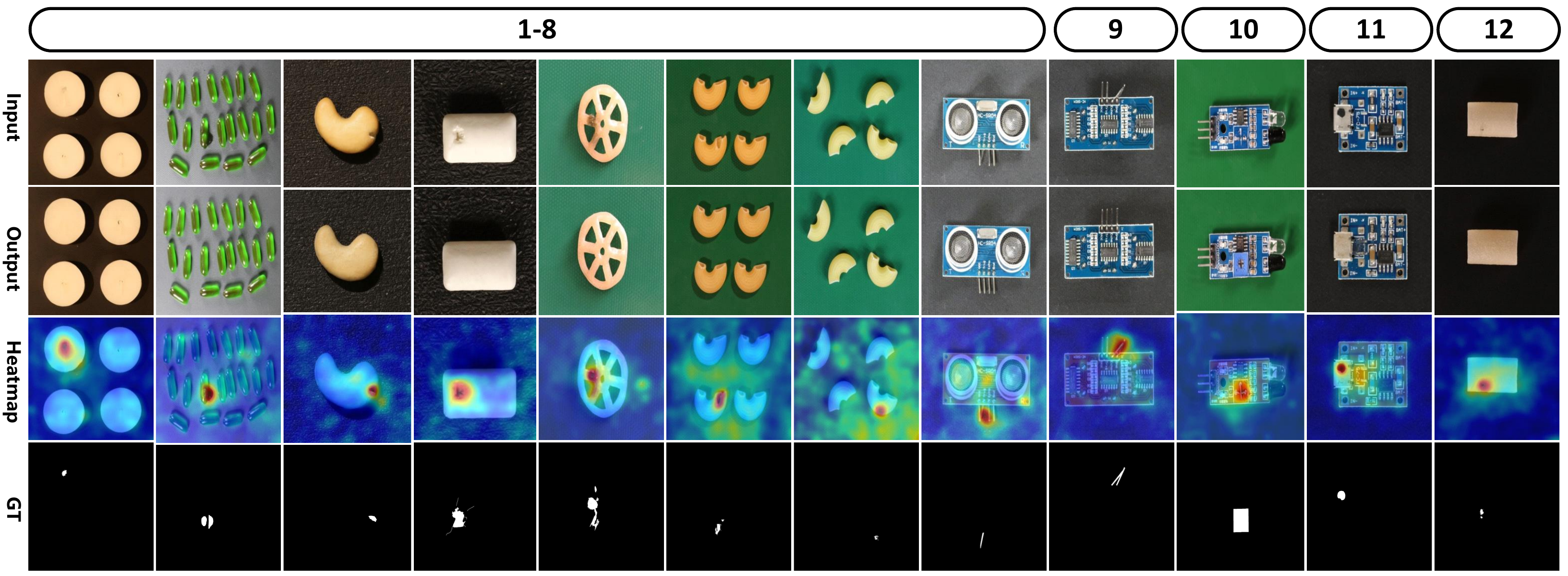}
    
   \caption{Qualitative comparison results under setting 7 of VisA, the numbers represent continual training classes.}
   \label{qua7}
\end{figure*}